\documentclass{article}[12pt]

\usepackage[margin=1.1in]{geometry}
\usepackage[utf8]{inputenc}

\usepackage{amsmath,amssymb,amsthm}
\usepackage{hyperref}
\hypersetup{colorlinks,allcolors=[rgb]{0,0.13672,0.95703}}
\usepackage{cleveref}
\usepackage[affil-sl]{authblk}
\usepackage{graphicx}
\usepackage{caption}
\usepackage{subcaption}
\usepackage{enumitem}
\usepackage{bbm}
\usepackage{wrapfig}
\usepackage[citestyle=numeric,natbib=true,bibstyle=numeric,maxbibnames=99]{biblatex}
\usepackage{makecell}

\newcommand\diag{{\operatorname{diag}}}

\newcommand\rk{{\mathrm{rank}}}

\renewcommand\l{{\mathcal{L}}}

\newcommand\R{{\mathbb{R}}}
\renewcommand\S{{\mathcal{S}}}
\newcommand\B{{\mathcal{B}}}
\newcommand\F{{\mathcal{F}}}

\renewcommand\v[1]{{\boldsymbol{#1}}}

\renewcommand\ln{{\mathrm{ln}}}

\def\rva{{\mathbf{a}}}

\def\rvu{{\mathbf{i}}}

\def\rvu{{\mathbf{u}}}
\def\rvv{{\mathbf{v}}}
\def\rvw{{\mathbf{w}}}
\def\rvx{{\mathbf{x}}}
\def\rvy{{\mathbf{y}}}
\def\rvz{{\mathbf{z}}}

\def\rmH{{\mathbf{H}}}

\def\rmU{{\mathbf{U}}}
\def\rmV{{\mathbf{V}}}
\def\rmW{{\mathbf{W}}}

\def\pls{{PL$^*$ }}
\def\plse{PL$^*_\epsilon$ }

\newtheorem{lemma}{Lemma}
\newtheorem{thm}{Theorem}
\newtheorem{prop}{Proposition}

 \theoremstyle{definition}

\newtheorem{defi}{Definition}

\newtheorem{remark}{Remark}
\newtheorem{assumption}{Assumption}
\newcommand{\secref}[1]{\hyperref[#1]{\textit{SI Appendix, section \ref{#1}}}}

\title{Loss landscapes and optimization in over-parameterized  non-linear systems and neural networks}

\author[a]{Chaoyue Liu}
\author[b,c]{Libin Zhu} 
\author[c]{Mikhail Belkin}

\affil[a]{Department of Computer Science and Engineering, The Ohio State University}
\affil[b]{Department of Computer Science and Engineering, University of California, San Diego} 
\affil[c]{Halicioğlu Data Science Institute, University of California, San Diego}

\begin{document}

\maketitle

\begin{abstract}
The  success of deep learning is due, to a large extent, to the remarkable effectiveness of gradient-based optimization methods applied to large neural networks. 
The purpose of this work is to propose a modern view and a general mathematical  framework for loss landscapes and efficient optimization  in over-parameterized machine learning models and systems of non-linear equations, a setting that includes over-parameterized deep neural networks. Our starting observation is that optimization problems corresponding to such systems are generally not convex, even locally. We argue that instead they satisfy PL$^*$, a variant of the  Polyak-{\L}ojasiewicz condition on most (but not all) of the parameter space, which guarantees both the existence of solutions 
and efficient optimization by (stochastic) gradient descent (SGD/GD).   The PL$^*$ condition of these systems is closely related to the condition number of the tangent kernel associated to a non-linear system showing how a PL$^*$-based non-linear theory parallels  classical analyses of over-parameterized linear equations. We show that wide neural networks satisfy the PL$^*$ condition, which explains the (S)GD convergence to a global minimum. Finally we propose a relaxation of the \pls condition applicable to ``almost'' over-parameterized systems.

\end{abstract}

\section{Introduction}




A singular feature of modern machine learning is a large  number of trainable model parameters. Just in the last few years we have seen state-of-the-art models grow from tens or hundreds of millions parameters to much larger systems with hundreds billions~\cite{brown2020language} or even trillions parameters~\cite{fedus2021switch}. Invariably these models are trained by gradient descent based methods, such as Stochastic Gradient Descent (SGD) or Adam~\cite{kingma2014adam}.  Why are these local gradient  methods so effective in optimizing  complex highly non-convex systems? In the past few years an emerging understanding of gradient-based methods have  started to focus on the insight that optimization dynamics of ``modern'' over-parameterized models with more parameters than constraints are very different from those of ``classical'' models when the number of constraints exceeds the number of parameters. 


The goal of this paper is to provide a modern view of the optimization landscapes, isolating key mathematical and conceptual elements that are essential for an optimization theory of over-parameterized models.



We start by characterizing a key difference between under-parameterized and over-parameterized landscapes. While both are generally non-convex, the nature of the non-convexity is rather different: under-parameterized landscapes are (generically) locally convex in a sufficiently small neighborhood of a local minimum. Thus classical analyses apply, if only locally, sufficiently close to a minimum. 
In contrast, over-parameterized systems are ``essentially'' non-convex systems, not even in  arbitrarily small neighbourhoods around global minima. 

Thus, we cannot expect the extensive theory of convexity-based analyses to apply to such over-parameterized problems. In contrast, we argue that these systems typically satisfy the Polyak-{\L}ojasiewicz  condition, or more precisely, its slightly modified variant -- PL$^*$ condition, on most (but not all) of the parameter space.  This condition ensures existence of solutions and convergence of GD and SGD, if it holds in a ball of sufficient radius. Importantly, we  show that sufficiently wide neural networks satisfy the PL$^*$ condition around their initialization point, thus guaranteeing convergence. 
In addition, we show how the PL$^*$ condition can be relaxed without significantly changing the analysis. We conjecture that many large systems behave as if the were over-parameterized along the stretch of their optimization path from initialization to the early stopping point.

In a typical supervised  learning task, given a training dataset of  size $n$, $\mathcal{D}=\{\rvx_i,y_i\}_{i=1}^n$, $\rvx_i \in \R^d, y\in\R$, and a parametric family of models $f(\rvw;\rvx)$, e.g., a neural network, one aims to find a  model with parameter vector $\rvw^*$,  that fits the training data, i.e., 
\begin{equation}
f(\rvw^*;\rvx_i) \approx y_i, \quad i = 1,2, \ldots, n.
\end{equation}

Mathematically, it is equivalent to solving, exactly or approximately, a system of $n$  equations\footnote{The same setup works for multiple outputs. For examples for multi-class classification problems both the prediction $f(\rvw;\rvx_i)$ and labels $\rvy_i$ (one-hot vector) are $c$-dimensional vectors, where $c$ is the number of classes. In this case, we are in fact solving $n\times c$  equations.  
Similarly, for multi-output regression with $c$ outputs, we have   $n\times c$ equations. }.  Aggregating them in a single map (and suppressing the dependence on the training data in the notation) we write:
\begin{equation}\label{eq:problem}
    \F(\rvw) = \rvy, ~~\text{where} ~\rvw\in \R^m, ~\rvy\in \R^n,~~\F(\cdot): \mathbb{R}^m \to \R^n.
\end{equation}
Here   $(\F(\rvw))_i := f(\rvw;\rvx_i)$.

 The system in Eq.(\ref{eq:problem}) is solved through minimizing a certain loss function $\l(\rvw)$, e.g.,  the square loss $$
 \l(\rvw) = \frac{1}{2}\|\F(\rvw)-\rvy\|^2 = \frac{1}{2}\sum_{i=1}^N (f(\rvw;\rvx_i)-y_i)^2
 $$
 constructed so  that the solutions of Eq.(\ref{eq:problem}) are global minimizers of $\l(\rvw)$. 
  This is a non-linear least squares problem, which is well-studied under classical under-parameterized settings (see~\cite{nocedal2006numerical}, Chapter 10).
 An exact solution of Eq.(\ref{eq:problem}) corresponds to interpolation, where a predictor fits the data exactly.

 As we discuss below, for over-parameterized systems ($m>n$), we expect exact solutions to exist.

 \begin{figure}
    \centering
    \begin{subfigure}[t]{0.45\textwidth}
        \centering
        \includegraphics[width=\linewidth]{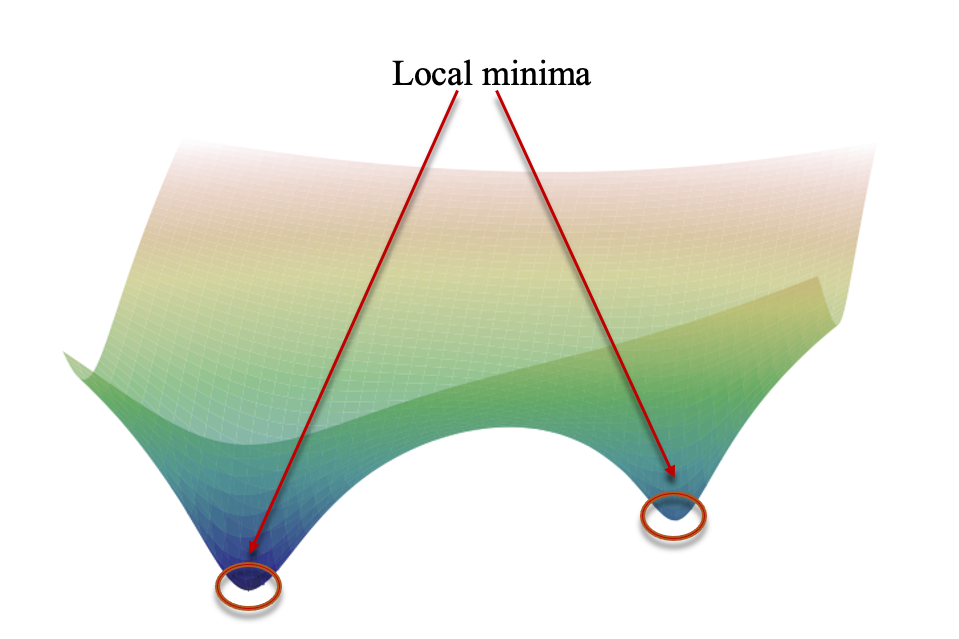} 
        \caption{Loss landscape of under-parameterized models} \label{fig:loc}
    \end{subfigure}
    \hfill
    \begin{subfigure}[t]{0.45\textwidth}
        \centering
        \includegraphics[width=\linewidth]{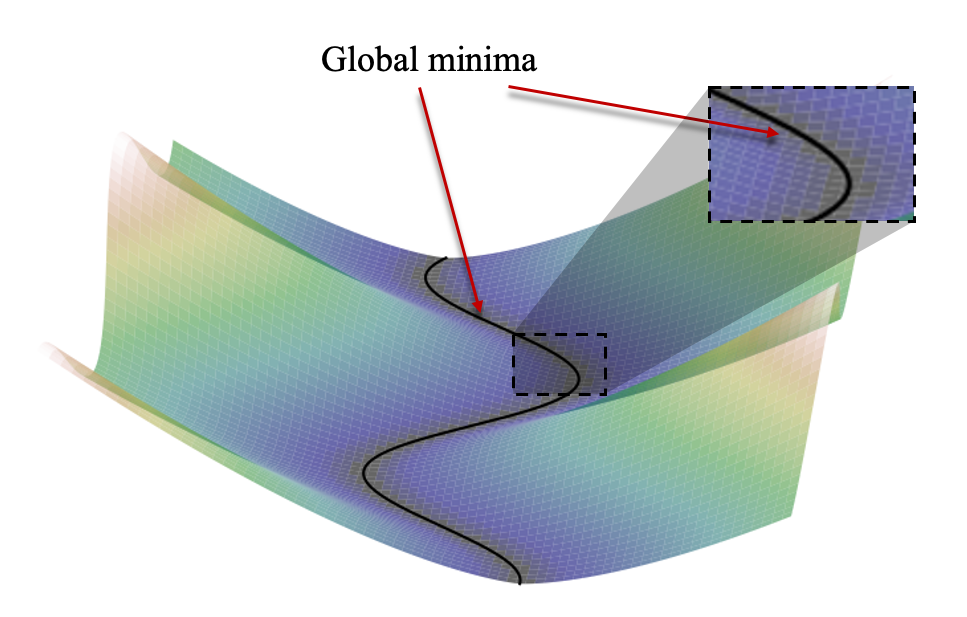} 
        \caption{Loss landscape of over-parameterized models} \label{fig:glob}
    \end{subfigure}
     \caption{Panel (a): Loss landscape is locally convex at local minima. Panel (b): Loss landscape incompatible with local convexity as the set of global minima is not locally linear. }\label{fig:loc&glob}
  \end{figure}


\paragraph{Essential non-convexity.} 
    Our starting point, discussed in detail in Section~\ref{sec:nonconvex}, is the observation that the loss landscape of an over-parameterized system is generally not convex in any neighborhood of any global minimizer. This is different from the case of under-parameterized systems, where the   loss landscape is globally not convex, but still typically locally convex, in a sufficiently small neighbourhood of a  (locally unique) minimizer.  In contrast, the set of solutions of over-parameterized systems is generically a manifold of positive dimension~\cite{cooper2018loss} (and indeed, systems large enough have no non-global minima~\cite{li2018over,nguyen2018loss,yu1995local}). Unless the solution manifold is linear (which is not generally the case) the landscape cannot be locally convex. The contrast between over and under-parametrization illustrated pictorially in Fig~\ref{fig:loc&glob}.

    The non-zero curvature of the curve of global minimizers at the bottom of the valley in  Fig~\ref{fig:loc&glob}(b) indicates the essential non-convexity of the landscape. In contrast, an under-parameterized landscape generally has multiple isolated local minima with positive definite Hessian of the loss, where the function is locally convex. Thus we conclude that
    \begin{center} {\it  
    Convexity is  not the right framework for analysis of over-parameterized systems, even locally. }
    \end{center}




    Without  assistance from local convexity, what  alternative mathematical framework can be used to analyze loss   landscapes and  optimization behavior of  non-linear over-parameterized systems? 
    
    In this paper we argue that such a general framework is provided by the Polyak-{\L}ojasiewicz condition~\cite{polyak1963gradient,lojasiewicz1963topological}, or, more specifically, by its variant that we call PL$^*$ condition (also called local PL condition in~\cite{oymak2019overparameterized}). We say that a non-negative function $\l$ satisfies $\mu$-PL$^*$ on a set ${\cal S}\subset \R^m$ for $\mu>0$, if 
    \begin{equation}
        \|\nabla \l (\rvw)\|^2 \ge \mu \l (\rvw), ~~\forall~\rvw \in \mathcal{S}.
    \end{equation}
    
    
   
   
  We will now outline some key reasons why PL$^*$ condition  provides a general framework for analyzing over-parameterized systems. In particular, we show 
  how it is satisfied by the loss functions of sufficiently wide neural networks, although which are certainly non-convex.

   
 \paragraph{PL$^*$ condition $\implies$ existence of solutions and exponential convergence of (S)GD.}
 The first key point (see Section~\ref{sec:pl-conv}), is that if $\l$ satisfies the $\mu$-PL$^*$ condition in a ball of radius $O(1/\mu)$ then $\l$ has a global minimum in that ball (corresponding to a solution of  
 Eq.(\ref{eq:problem})). Furthermore, (S)GD initialized at the center of such a ball\footnote{The constant  in $O(1/\mu)$ is different for GD and SGD.} converges exponentially to a global minimum of $\l$. 
 Thus to establish both existence of solutions to Eq.(\ref{eq:problem}) and efficient optimization, it is sufficient to verify the PL$^*$ condition in a ball of a certain radius. 
 
 \paragraph{Analytic form of PL$^*$ condition via the spectrum of the tangent kernel.}
 
 Let 
 $D \F(\rvw)$ be the differential of the map $\F$ at $\rvw$, viewed as a $n\times m$ matrix.
The tangent kernel of $\F$ is defined as an $n \times n$ matrix
 $$
 K(\rvw)= D\F(\rvw)\,D\F^T(\rvw).
 $$
 It is clear that $K(\rvw)$ is a positive semi-definite matrix.  It can be seen (Section~\ref{sec:gen_theory}) that square loss $\l$ is $\mu$-PL$^*$ at $\rvw$, where 
 \begin{equation}\label{eq:lambda_muPL}
 \mu = \lambda_{min}(K(\rvw)),
 \end{equation} 
 is the smallest eigenvalue of the kernel matrix.  
 Thus verification of the PL$^*$ condition reduces to analyzing the spectrum of the tangent kernel matrix associated to $\F$. 
 
 It is worthwhile to compare this to the standard analytical condition for convexity, requiring that the Hessian of the loss function, $H_{\l}$, is positive definite. While these spectral conditions appear to be  similar, the similarity is superficial as $K(\rvw)$ contains first order derivatives, while the Hessian is a second order operator. Hence, as we discuss below, the tangent kernel and the Hessian have very different properties under coordinate transformations and in other settings.




 \paragraph{Why PL$^*$  holds across most of the parameter space for over-parameterized systems.}
 
 We will now discuss the intuition why $\mu$-PL$^*$ for a sufficiently small $\mu$ should  be expected to hold in the domain containing most (but not all) of the parameter space $\R^m$. The intuition is based on simple parameter counting. For the simplest example  consider the case $n=1$.  In that case 
 the tangent kernel is a scalar and
 $K(\rvw) = \|D \F (\rvw)\|^2$  is singular if and only if $\|D \F(\rvw)\|=0$. Thus, by parameter counting, we expect  the singular set $\{\rvw: K(\rvw)=0\}$  to have co-dimension $m$ and thus, generically, consist of isolated points. Because of the Eq.(\ref{eq:lambda_muPL})  
 we generally expect most points sufficiently removed from the singular set to satisfy the PL$^*$ condition.  By a similar parameter counting argument, the singular set of $\rvw$, such that $K(\rvw)$ is not full rank will of co-dimension $m-n+1$. As long as $m>n$, we expect the surroundings of the singular set, where the PL$^*$ condition does not hold,  to be ``small'' compared to the totality of the parameter space. Furthermore, the larger the degree of the model over-parameterization ($m-n$) is, the smaller the singular set is expected to be. This intuition is represented graphically in Fig.~\ref{fig:pldomain_main}.
 
 Note that under-parameterized systems are always rank deficient and have $\lambda_{min}(K(\rvw))\equiv 0$. Hence such systems never satisfy PL$^*$.
  \begin{figure}[t]\
        \centering
        \includegraphics[width=0.9\linewidth]{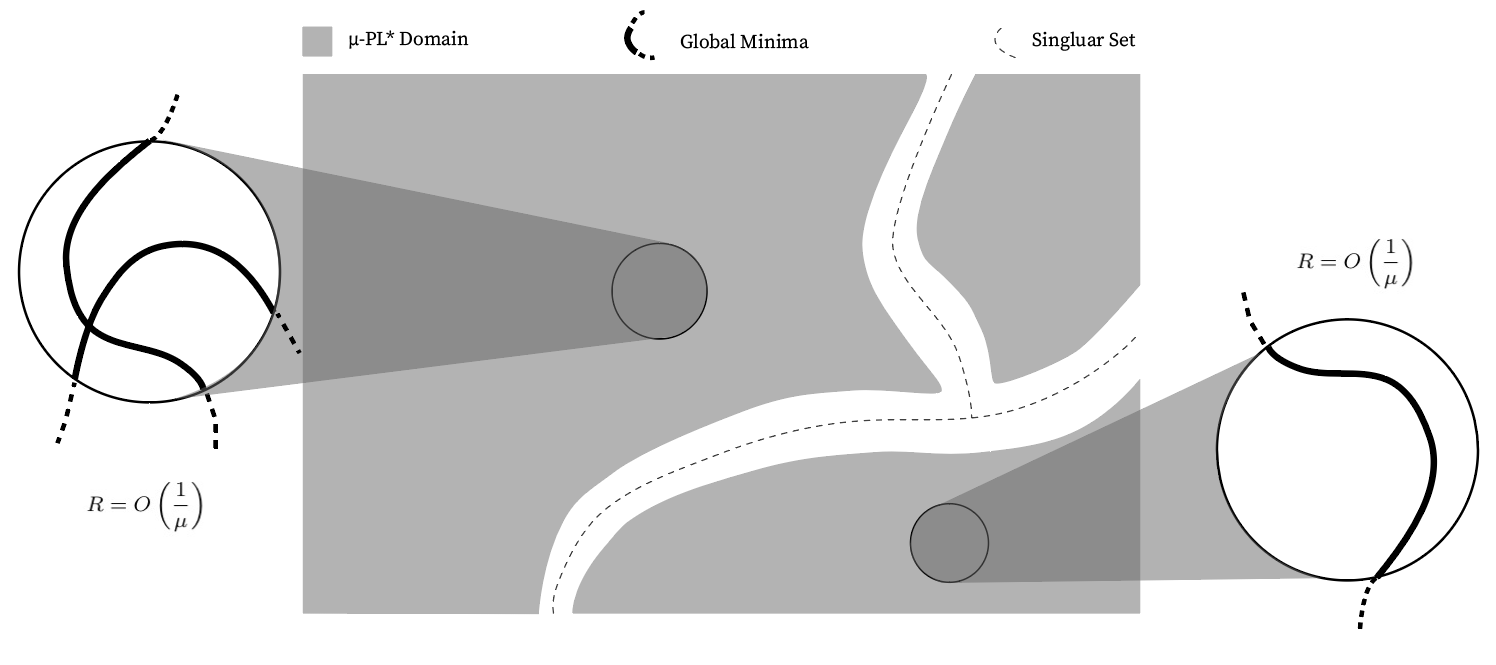} 
        \caption{The loss function $\l(\rvw)$ is $\mu$-PL$^*$ inside the shaded domain. Singular set corresponds to parameters $\rvw$ with degenerate tangent kernel $K(\rvw)$. Every ball of radius $R=O(1/\mu)$ within the shaded set intersects with the set of global minima of $\l(\rvw)$, i.e., solutions to $\F(\rvw)=\rvy$. } \label{fig:pldomain_main}
  \end{figure}
 
 \paragraph{Wide neural networks satisfy PL$^*$ condition.} 

We show that wide neural networks, as special cases of over-parameterized models, are PL$^*$, using the technical tools developed in Section~\ref{sec:gen_theory}. This property of the neural networks is a key step to understand the success of the gradient descent based optimization, as seen in Section~\ref{sec:pl-conv}.
To show the PL$^*$ condition for wide neural networks, a powerful, if somewhat crude, tool is provided by the remarkable recent discovery~\cite{jacot2018neural} that tangent kernels (so-called NTK) of wide neural networks  with linear output layer are nearly constant in a ball $\B$ of a certain radius around the ball with a random center.  More precisely, it can be shown that the norm of the Hessian tensor $\|\rmH_\F(\rvw)\| = O^*(1/\sqrt{m})$, where $m$ is the width of the neural network and $\rvw \in \B$~\cite{liu2020linearity}.

Furthermore (see Section~\ref{subsec:uc_ball}): 
$$
|\lambda_{min}(K(\rvw)) - \lambda_{min}(K(\rvw_0))| <O\left(\sup_{\rvw \in \B} \|\rmH_\F(\rvw)\|\right) = O(1/\sqrt{m}).
$$
Bounding the kernel eigenvalue at the initialization point $\lambda_{min}(K(\rvw_0))$ from below (using the results from~\cite{du2018gradientshallow,du2018gradientdeep}) completes the analysis. 

To prove the result for general neural networks (note that the Hessian norm is typically large for such networks~\cite{liu2020linearity}) we observe that they can be decomposed as a composition of a network with a linear output layer and a coordinate-wise non-linear transformation of the output. The \pls condition is preserved under well-behaved non-linear transformations which yields the result. 


 
 
   \begin{wrapfigure}[17]{r}{6cm}
  \vspace{-6pt}
\begin{center}
  \includegraphics[width=0.4\textwidth]{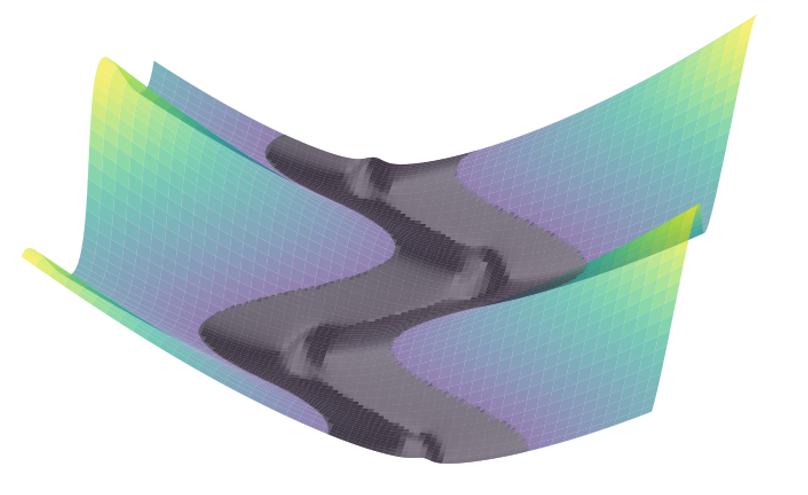} 
\end{center}
\vspace{-15pt}
\caption{The loss landscape of  ``almost over-parameterized'' systems. The landscape looks  over-parameterized  except for the grey area where the loss is small. Local minima of the loss are contained there.}
\label{fig:pl-epsilon-intro}
\vspace{-36pt} 
\end{wrapfigure}

 \paragraph{Relaxing the PL$^*$ condition: when the systems are almost over-parameterized.}

 Finally, a framework for understanding large systems cannot be complete without addressing the question of transition between under-parameterized and over-parameterized systems. While neural network models such as those in~\cite{brown2020language,fedus2021switch} have billions or trillions on parameters, they are often trained on very large datasets and thus have comparable number of constraints. Thus in the process of optimization they may initially appear over-parameterized, while no truly interpolating solution may exist. Such an optimization landscape is shown in Fig.~\ref{fig:pl-epsilon-intro}.   
 While an in-depth analysis of this question is beyond the scope of this work, we propose a version of the PL$^*$ condition that can account for such behavior by postulating that the PL$^*$ condition is satisfied for values of $\l(\rvw)$ above a certain threshold and that the optimization path from the initialization to the early stopping point lies in the PL$^*$ domain. 
 Approximate convergence of (S)GD can be shown under this condition, see Section~\ref{subsec:pl-epsilon}.

\subsection{Related work}
Loss landscapes of over-parameterized machine learning models, especially deep neural networks, have  recently been studied in a number of papers including~\cite{cooper2018loss, li2018over,nguyen2018loss,yu1995local,lederer2020no,mei2018mean}. The work~\cite{cooper2018loss} suggests that the solutions of an over-parameterized system typically are not discrete and form a lower dimensional manifold in the parameter space. In particular~\cite{li2018over,nguyen2018loss}
show that sufficiently over-parameterized neural networks have no strict local minima under certain mild assumptions. Furthermore in~\cite{lederer2020no} it is proved that, for sufficiently over-parameterized neural networks, each local minimum (if it exists) is ``path connected'' to a global minimum where the loss value on the path is non-increasing. While the above works mostly focus on the properties of the minima of the loss landscape, in this paper we  on the landscape within neighborhoods of these minima, which can often cover the whole optimization path of gradient-based methods. 

There has also been a rich literature that particularly focuses on the optimization of wide neural networks using gradient descent~\cite{soltanolkotabi2018theoretical, du2018gradientshallow,du2018gradientdeep,lee2019wide,bartlett2019gradient,arora2019fine,ji2019polylogarithmic,oymak2019towards}, and  SGD~\cite{allen2019convergence,zou2018stochastic},   especially after the discovery of the constancy of neural tangent kernels (NTK) for certain neural networks~\cite{jacot2018neural}. Many of these analyses are based on the observation that the constancy of the tangent kernel implies that the training dynamic of these networks is approximately that of a linear model~\cite{lee2019wide}.  This type of analysis can be viewed within our framework of Hessian norm control.

The Polyak-{\L}ojasiewicz (PL) condition has attracted interest in connection with  convergence of neural networks and other non-linear systems as it allows to prove some key properties of convexity with respect to the optimization in non-convex settings. 
For example, the work~\cite{charles2018stability} proved that a composition of strongly convex function with a one-layer leaky ReLU network satisfies the PL condition.
The paper~\cite{bassily2018exponential} 
proved fast convergence of stochastic gradient descent with constant step size for over-parameterized models, 
 while~\cite{oymak2019overparameterized}
refined this result showing SGD convergence within a ball with a certain radius. In related work~\cite{gupta2019path} shows that the optimization path length can be upper bounded. 

Finally, it is interesting to point out the connections with the contraction theory for differential equations explored in~\cite{wensing2020beyond}.

\section{Notation and standard definitions}\label{sec:prelim}

We use bold lowercase letters, e.g., $\rvv$, to denote vectors,  capital letters, e.g., $W$, to denote matrices, and bold capital letters, e.g., $\rmW$, to denote tuples of matrices or higher order tensors.
We denote the set $\{1,2,\cdots, n\}$ as $[n]$.  Given a function $\sigma(\cdot)$, we use $\sigma'(\cdot)$ and $\sigma''(\cdot)$ to denote its first and second respectively. For vectors, we use $\|\cdot \|$ to denote the Euclidean norm and $\|\cdot\|_{\infty}$ for the $\infty$-norm. For matrices, we use $\|\cdot\|_F$ to denote the Frobenius norm and $\|\cdot\|_2$ to denote the spectral norm (i.e., 2-norm).
We use $D \F$ to represent the differential map of $\F:\R^m\to\R^n$.  Note that $D \F$ is represented as
 a $n\times m$ matrix, with $(D \F)_{ij} := \frac{\partial \F_i}{\partial w_j}$. 
We denote Hessian of the function $\F$ as $\rmH_\F$, which is an $n \times m \times m$ tensor with $(\rmH_\F)_{ijk} = \frac{\partial^2 \F_i}{\partial w_j\partial w_k}$, and define the norm of the Hessian tensor to be the maximum of the spectral norms of each of its Hessian components, i.e., $\|\rmH_\F\| = \max_{i\in [n]} \|H_{\F_i}\|_2$, where $H_{\F_i}=\partial^2 \F_i/\partial \rvw^2$. When necessary, we also assume that $\rmH_\F$ is continuous. We also denote the Hessian matrix of the loss function as $H_{\mathcal{L}} := \partial^2 \mathcal{L}/\partial \rvw^2$, which is an $m\times m$ matrix. We denote the smallest eigenvalue of a matrix $K$ as $\lambda_{min}(K)$.


In this paper, we consider the problem of solving a system of equations of the form Eq.(\ref{eq:problem}), i.e. finding $\rvw$, such that $\F(\rvw)=\rvy$. 
This problem is  solved by minimizing a loss function $\l(\F(\rvw),\rvy)$, such as the square loss $\l(\rvw) =\frac{1}{2}\|\F(\rvw)-\rvy\|^2$,  with gradient-based algorithms. Specifically, 
the gradient descent method starts from the initialization point $\rvw_0$ and updates the parameters as follows:
\begin{equation}
\rvw_{t+1} = \rvw_{t}-\eta \nabla \l(\rvw_t),~~~ \forall t \in \mathbb{N}.
\end{equation} 
We call the set $\{\rvw_t\}_{t=0}^{\infty}$ the optimization path, and put  $\rvw_{\infty}=\lim_{t\to \infty}\rvw_t$ (assuming the limit exists).

\smallskip
\smallskip
Throughout this paper, we assume the map $\F$ is 
 Lipschitz continuous and smooth. See the definition below:
\begin{defi}[Lipschitz continuity]
Map $\F: \mathbb{R}^m \to \mathbb{R}^n$ is $L_\F$-Lipschitz continuous, if \begin{equation}\|\F(\rvw)-\F(\rvv)\| \le L_\F\|\rvw-\rvv\|, \quad \forall \rvw, \rvv \in \mathbb{R}^m.\end{equation}
\end{defi} 
\begin{remark}
In supervised learning cases, $(\F(\rvw))_i = f(\rvw;\rvx_i)$. If we denote 
$L_f$ as the Lipschitz continuity of the machine learning model $f(\rvw;\rvx)$ w.r.t. the parameters $\rvw$, then $\|\F(\rvw)-\F(\rvv)\|^2 = \sum_i (f(\rvw;\rvx_i) - f(\rvv;\rvx_i))^2 \le n L_f^2\|\rvw-\rvv\|^2.$
One has $L_\F \le \sqrt{n}L_f$.
\end{remark}
A direct consequence of the $L_\F$-Lipschitz condition is that $\|D \F(\rvw)\|_2 \le L_\F$ for all $\rvw\in\mathbb{R}^m$. Furthermore, it is easy to see that the tangent kernel has bounded spectral norm:
\begin{prop}\label{prop:lf}
If map $\F$ is $L_\F$-Lipschitz, then $\|K(\rvw)\|_2 \le L_\F^2$. 
\end{prop}
\begin{defi}[Smoothness]
Map $\F: \mathbb{R}^m \to \mathbb{R}^n$ is $\beta_\F$-smooth, if 
\begin{equation}
    \|\F(\rvw) - \F(\rvv) -D \F(\rvv)(\rvw-\rvv)\|_2 \le \frac{\beta_\F}{2}\|\rvw-\rvv\|^2, \quad \forall \rvw, \rvv \in \mathbb{R}^m.
\end{equation}
\end{defi}

\paragraph{Jacobian matrix.} Given a map $\Phi: \rvz\in\mathbb{R}^n \mapsto \Phi(\rvz)\in \mathbb{R}^n$, the corresponding Jacobian matrix $J_{\Phi}$ is defined as following: each entry
\begin{equation}
    (J_{\Phi})_{ij} := \frac{\partial \Phi_i}{\partial z_j}.
\end{equation}

\section{Essential non-convexity of loss landscapes of over-parameterized non-linear systems}
\label{sec:nonconvex}

In this section we discuss  the observation that  over-parameterized systems give rise to landscapes that are {\it essentially} non-convex -- there 
typically exists no neighborhood around any global minimizer where the loss landscape is convex. This is in contrast to under-parameterized systems where 
such a neighborhood typically exists, although it can be small. 

For over-parameterized system of equations, the number of parameters $m$ is larger than the number of constraints $n$. In this case, the system of equations generically has exact solutions, which form a {\it continuous} manifold (or several {\it continuous} manifolds), typically of dimension $m-n>0$ so that  none of the solutions are isolated. A specific result for wide neural networks, showing non-existence of isolated global minima of the loss function is given in Proposition~\ref{prop:no-iso} (Appendix~\ref{secapp:no_isolated}).

It is important to note that such continuous manifolds of solutions generically have non-zero curvature\footnote{Both extrinsic,  the second fundamental form,  and intrinsic Gaussian curvature.}, due to the non-linear nature of the system of equations. This results in the lack of local convexity of the loss landscape, i.e., loss landscape is nonconvex in any neighborhood of a solution (i.e, a global minimizer of $\l$). This can be seen from the following argument. 
Suppose that the loss function landscape of $\l$ is convex within  a ball $\B$ that intersects the set of global minimizers $\mathcal M$. The  minimizers of a convex function within a convex domain form a convex set, thus $\mathcal{M} \cap \mathcal S$ must also be convex. Hence $\mathcal{M} \cap \mathcal S$ must be a convex subset within a lower dimensional linear subspace of $\R^m$ and cannot have curvature (either extrinsic ot intrinsic). This geometric intuition is illustrated in Fig.~\ref{fig:glob}, where the set of global minimizers is of dimension one.




To provide an alternative analytical intuition (leading to a precise argument)  
 consider an over-parameterized system $\F(\rvw):\mathbb{R}^m \rightarrow \mathbb{R}^n$, where $m > n$, with the square loss function $\l(F(\rvw),\rvy) = \frac{1}{2}\| \F(\rvw)-\rvy\|^2$.  The Hessian matrix of the loss function takes the form
\begin{align}\label{eq:hessianofloss} 
H_\l(\rvw) = \underbrace{D \F(\rvw)^T \frac{\partial^2\l}{\partial \F^2}(\rvw)D \F(\rvw)}_{A(\rvw)} +\underbrace{\sum_{i=1}^n\left(\F(\rvw)-\rvy\right)_i H_{\F_i}(\rvw)}_{B(\rvw)} , 
\end{align}
where $H_{\F_i}(\rvw) \in \mathbb{R}^{m\times m}$ is the Hessian matrix of $i$th output of $\F$ with respect to $\rvw$.

Let $\rvw^*$ be a solution to Eq.(\ref{eq:problem}). Since  $\rvw^*$ is a global minimizer of $\l$,  $B(\rvw^*)=0$.
We note that  $A(\rvw^*)$ is a positive semi-definite matrix of rank at most $n$ with at least $m-n$ zero eigenvalues.

Yet, in a neighbourhood of $\rvw^*$ there typically are points where $B(\rvw)$ has rank $m$. As we show below, this observation,  together with a mild technical assumption on the loss, implies 
that $H_\l(\rvw)$ cannot be positive semi-definite in any ball around $\rvw^*$ and hence is not locally convex. 
To see why this is the case, consider an example of a system with only one equation ($n=1)$. The loss function  takes the form as $\l(\rvw) = \frac{1}{2}( \F(\rvw) - y)^2, y\in \R$ and the Hessian of the loss function can be written  as 
\begin{align}
    H_\l(\rvw) = D \F(\rvw)^T D \F(\rvw) + \sum_{i}((\F(\rvw))_i - y_i) H_{\F_i}(\rvw).
\end{align}
Let $\rvw^*$ be a solution, $\F(\rvw^*) = y$ and suppose that $D\F(\rvw^*)$ does not vanish.
{ In the neighborhood of $\rvw^*$, there exist arbitrarily close points $\rvw^* + \boldsymbol{\delta}$ and $\rvw^* - \boldsymbol{\delta}$, such that $\F(\rvw^* + \boldsymbol{\delta}) - y >0$ and $\F(\rvw^* - \boldsymbol{\delta}) - y <0$}. 
Assuming that the rank of $H_{\F_i}(\rvw^*)$ is greater than one, and 
noting that the rank  of $D \F(\rvw)^T D \F(\rvw)$ is at most one, it is easy to see that either $H_\l(\rvw^*+\boldsymbol{\delta})$ or $H_\l(\rvw^*-\boldsymbol{\delta})$ must have negative eigenvalues, which rules out local convexity at $\rvw^*$. \\
A more general version of this argument is given in the following:
\begin{prop}[Local non-convexity]\label{prop:nolocalconx} Let  $\l(\rvw^*) = 0$ and,
furthermore, 
assume that $D\F(\rvw^*) \ne 0$
and  $\mathrm{rank}(H_{\F_i}(\rvw^*))>2n$ for at least one $i\in [n]$. Then $\l(\rvw)$ is not convex in any neighborhood of $\rvw^*$.
\end{prop}

\begin{remark}
 Note that in general we do not expect $D \F$ to vanish at $\rvw^*$ as there is no reason why a solution to Eq.(\ref{eq:problem}) should be a critical point of $\F$. For a general loss $\l(\rvw)$,  the assumption in Proposition~\ref{prop:nolocalconx} that $D \F(\rvw^*) \ne 0$ is replaced by $\left(\frac{d}{d\rvw}\frac{\partial\l}{\partial \F}\right)(\rvw^*) \ne 0$. 
\end{remark}
\noindent
A full proof of Prop.~\ref{prop:nolocalconx} for a general loss function can be found in Appendix~\ref{prf:nolocalconx}.

\paragraph{Comparison to under-parameterized systems.}   For under-parameterized systems, local minima are generally isolated, as illustrated in Figure~\ref{fig:loc}. Since $H_\l(\rvw^*)$ is generically positive definite when $\rvw^*$ is an isolated local minimizer, by the continuity of $H_\l(\cdot)$, positive definiteness holds in the neighborhood of $\rvw^*$. Therefore, $\l(\rvw)$ is locally convex around $\rvw^*$. \\

\section{Over-parameterized non-linear systems are  PL$^*$ on most of the parameter space}
\label{sec:gen_theory}
In this section we argue that loss landscapes of  over-parameterized systems satisfy the PL$^*$ condition across most of their parameter space. 

We begin with a sufficient analytical condition {\it uniform conditioning} of a system closely related to PL$^*$.



\begin{defi}[Uniform conditioning]
We say that  $\F(\rvw)$ is $\mu$-uniformly conditioned ($\mu>0$) on $\mathcal{S}\subset \R^m$
 if the smallest eigenvalue of its tangent kernel $K(\rvw)$ satisfies  
\begin{equation}\label{eq:lmin}
    \lambda_{min}(K(\rvw)) \ge \mu,\ \forall \rvw\in \mathcal{S}.
\end{equation}
\end{defi}
\noindent The following important connection shows that uniform conditioning of the system is sufficient for the corresponding square loss landscape to satisfy the PL$^*$ condition.
\begin{thm}
[Uniform conditioning  $\Rightarrow$ PL$^*$ condition]\label{thm:wellconditionandpl}
If $\F(\rvw)$ is $\mu$-uniformly conditioned,  on a set $\mathcal{S}\subset \mathbb{R}^m$, then the square loss function $\l(\rvw)=\frac{1}{2}\|\F(\rvw)-\rvy\|^2$ satisfies $\mu$-PL$^*$ condition on $\mathcal{S}$. 
\end{thm}
\begin{proof}
\begin{align*}
    \frac{1}{2}\Vert \nabla \l(\rvw) \Vert^2 =   \frac{1}{2}( \F(\rvw)-\rvy)^TK(\rvw) ( \F(\rvw)-\rvy)
    \geq \frac{1}{2}\lambda_{min}(K(\rvw)) \Vert \F(\rvw)-\rvy\Vert^2 
    = \lambda_{min}(K(\rvw))\l(\rvw) 
    \ge \mu \l(\rvw).
\end{align*}
\end{proof}

We will now provide some intuition for why we expect $\lambda_{min}(K(\rvw))$ to be separated from zero over most but not all of the optimization landscape. 
The key observation is that 
$$
\rk\,K(\rvw) = \rk\, \left(D\F(\rvw)\,D\F^T(\rvw)\right) =\rk\, D \F (\rvw)
$$
Note that kernel $K(\rvw)$ is an $n\times n$ positive semi-definite matrix by definition. Hence the {\it singular set} $\S_{sing}$, where the tangent kernel is degenerate,  can be written as 
$$
\S_{sing} = \{\rvw \in \R^m\,|\, \lambda_{min}(K(\rvw)) = 0\} = \{\rvw\in\R^m\,|\, \rk\, D \F (\rvw) <n\}.
$$
Generically, $\rk\, D \F (\rvw) = \min (m,n)$. Thus for over-parameterized systems, when $m > n$,  we expect $\S_{sing}$ to have  positive codimension and to be a set of measure zero. In contrast, in under-parametrized settings, $m<n$ and the tangent kernel is always degenerate, $\lambda_{min}(K(\rvw)) \equiv 0$ so such systems {\it cannot} be uniformly conditioned according to the definition above.
Furthermore, while Eq.(\ref{eq:lmin}) provides a sufficient condition, it's also in a certain sense necessary: 
\begin{prop} If $\lambda_{min}(K(\rvw_0)) = 0$
then the system $\F(\rvw) = \rvy$ cannot satisfy PL$^*$ condition for all $\rvy$ on any set $\S$ containing $\rvw_0$. 
\end{prop}
\begin{proof}
Since $\lambda_{min}(K(\rvw_0)) = 0$, the matrix 
$K(\rvw_0)$ has a non-trivial null-space. Therefore we can choose $\rvy$ so that  $K(\rvw_0)(\F(\rvw_0)-\rvy)=0$ and $\F(\rvw_0)-\rvy\ne 0$.  We have 
$$
\frac{1}{2}\Vert \nabla \l(\rvw_0) \Vert^2 =   \frac{1}{2}( \F(\rvw_0)-\rvy)^TK(\rvw_0) ( \F(\rvw_0)-\rvy)=0
$$
and hence the PL$^*$ condition cannot be satisfied at $\rvw_0$.
\end{proof}
Thus we see that only over-parameterized systems can be PL$^*$, if we want that condition to hold for any label vector $\rvy$.

By parameter counting, it is easy to see that the codimension of the singular set $\S_{sing}$ is expected to be $m-n+1$. Thus, on a compact set, for $\mu$ sufficiently small, points which are not $\mu$-PL$^*$ will be found around $\S_{sing}$, a low-dimensional subset of $\R^m$. 
This is represented graphically in Fig.~\ref{fig:pldomain}. 
Note that the more over-parameterization we have, the larger the codimension is expected to be.

To see a particularly simple example of this phenomenon, consider the case when $D\F(\rvw)$ is a random matrix with Gaussian entries\footnote{In this the matrix ``$D\F(\rvw)$'' does not have to correspond to an actual map $\F$, rather the intention is to illustrate how over-parameterization leads to better conditioning.}. Denote by 
$$
\kappa=\frac{\lambda_{max}(K(\rvw))} {\lambda_{min}(K(\rvw))}
$$ the condition number of $K(\rvw)$. Note that by definition $\kappa \ge 1$. 
It is shown in~\cite{chen2005condition} that (assuming $m>n$)
$$
\mathbb{E}(\log \kappa) < 2 \log \frac{m}{m-n+1} +5.
$$
We see that as the amount of over-parameterization increases $\kappa$ converges in expectation (and also can be shown with high probability)  to a small constant.
While this example is rather special, it is representative of the concept that over-parameterization helps with conditioning. In particular, as we discuss below in Section~\ref{sec:wide_nn_pl*}, wide neural networks satisfy the PL$^*$ with high probability within a random ball of a constant radius.

     \begin{figure}[t]
        \centering
        \includegraphics[width=0.6\linewidth]{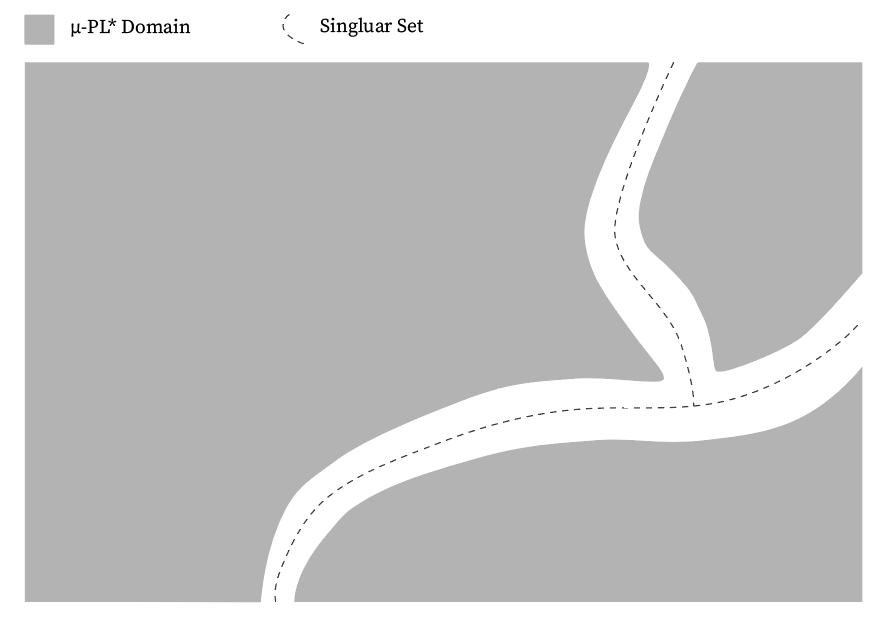} 
        \caption{$\mu$-PL$^*$ domain and the singular set. We expect points away from the singular set to satisfy $\mu$-PL$^*$ condition for sufficiently small $\mu$. } \label{fig:pldomain}
  \end{figure}

We will now provide  techniques for proving that the PL$^*$ condition holds for specific systems and later in Section~\ref{sec:wide_nn_pl*} we will show how these techniques apply to wide neural networks,  demonstrating that they are PL$^*$. In Section~\ref{sec:pl-conv} we discuss the implications of the PL$^*$ condition for the existence of solutions and convergence of (S)GD, in particular for deep neural networks.

\subsection{Techniques for establishing the PL$^*$ condition} \label{subsec:uc_ball}

While we expect typical over-parameterized systems to satisfy the PL$^*$ condition at most points, directly analyzing the smallest eigenvalue of the corresponding kernel matrix is often difficult. Below we describe two methods for demonstrating the PL$^*$ condition holds in a ball of a certain radius. First method is based on the  observation that is well-conditioned in a ball provided it is well-conditioned at its center and the Hessian norm is not large compared to the radius of the ball. 
Interestingly,   this ``Hessian control'' condition holds for a broad class of non-linear systems. In particular, as discussed in Sec~\ref{sec:wide_nn_pl*}, wide neural networks with linear output layer have small Hessian norm. For some intuition on why this appears to be a feature of many large systems see Appendix~\ref{appsec:small_hessian}.


The second approach to demonstrating conditioning is by noticing that it is preserved under well-behaved transformations of the input or output or when composing certain models. 

Combining these methods yields a general result on PL$^*$ condition for wide neural networks in Section~\ref{sec:wide_nn_pl*}.

\paragraph{Uniform conditioning through Hessian spectral norm.}
We will now show that controlling the norm of the Hessian tensor for the map $\F$ leads to well-conditioned systems. 
The idea of the analysis is that  the change of the tangent kernel of $\F(\rvw)$ can be bounded in terms of the norm of the Hessian of $\F(\rvw)$. Intuitively, this is 
analogous to the mean value theorem, bounding the  first derivative of $\F$ by its second derivative.
If the Hessian norm is sufficiently small, the change of the tangent kernel and hence its conditioning can be controlled in a ball $B(\rvw_0,R)$ with a finite radius, as long as the tangent kernel matrix at the center point $K(\rvw_0)$ is well-conditioned.


\begin{thm}[PL$^*$ condition via Hessian Norm]\label{thm:ntk}Given a point $\rvw_0\in \mathbb{R}^m$, suppose the tangent kernel matrix $K(\rvw_0)$ is strictly positive definite, i.e., $\lambda_0:=\lambda_{min}(K(\rvw_0))>0$. If the Hessian spectral norm $\|\rmH_\F\|\le \frac{\lambda_0-\mu}{2L_\F\sqrt{n}R}$ holds within the ball $B(\rvw_0,R)$ for some $R>0$ and $\mu>0$,  then the tangent kernel $K(\rvw)$ is $\mu$-uniformly conditioned in the ball $B(\rvw_0,R)$. Hence, the square loss $\l(\rvw)$ satisfies the $\mu$-PL$^*$ condition in $B(\rvw_0,R)$.
\end{thm}

\begin{proof}
First, let's consider the difference between the tangent kernel matrices at $\rvw\in B(\rvw_0,R)$ and at $\rvw_0$.

By the assumption, we have, for all $\rvv\in B(\rvw_0,R)$, $\|\rmH_\F(\rvv)\|< \frac{\lambda_0-\mu}{2L_\F\sqrt{n}R}$. Hence, for each $i\in [n]$, $\|H_{\F_i}(\rvw)\|_{2} < \frac{\lambda_0-\mu}{2L_\F\sqrt{n}R}$.
Now, consider an arbitrary point $\rvw\in B(\rvw_0,R)$. For all $i\in [n]$, we have:
\begin{equation}
    D \F_i(\rvw) = D \F_i(\rvw_0) + \int_{0}^{1} H_{\F_i}(\rvw_0+\tau (\rvw-\rvw_0)) (\rvw-\rvw_0) d\tau.
\end{equation}
Since $\tau$ is in $[0,1]$, $\rvw_0+\tau (\rvw-\rvw_0)$ is on the line segment $S(\rvw_0,\rvw)$ between $\rvw_0$ and $\rvw$, which is inside of the ball $B(\rvw_0,R)$. Hence,
\begin{eqnarray*}
\|D \F_i(\rvw)-D \F_i(\rvw_0)\| 
\le \sup_{\tau \in [0,1]}\|H_{\F_i}(\rvw_0+\tau (\rvw-\rvw_0))\|_{2}\cdot \|\rvw-\rvw_0\|
\le \frac{\lambda_0-\mu}{2L_\F\sqrt{n}R}\cdot R
= \frac{\lambda_0-\mu}{2L_\F\sqrt{n}}.
\end{eqnarray*}
In the second inequality above, we used the fact that $\|H_{\F_i}\|_{2} < \frac{\lambda_0-\mu}{2L_\F\sqrt{n}R}$ in the ball $B(\rvw_0,R)$. Hence,
\begin{equation*}
    \|D \F(\rvw)-D \F(\rvw_0)\|_F = \sqrt{\sum_{i\in [n]}\|D \F_i(\rvw)-D \F_i(\rvw_0)\|^2} \le \sqrt{n}\frac{\lambda_0-\mu}{2L_\F\sqrt{n}} = \frac{\lambda_0-\mu}{2L_\F}.
\end{equation*}
Then, the spectral norm of the tangent kernel change is bounded by
\begin{eqnarray*}
\|K(\rvw)-K(\rvw_0)\|_2 &=& \|D \F(\rvw)D \F(\rvw)^T-D \F(\rvw_0)D \F(\rvw_0)^T\|_2\\
&=&\|D \F(\rvw)\left(D \F(\rvw)-D \F(\rvw_0)\right)^T + \left(D \F(\rvw)-D \F(\rvw_0)\right)D \F(\rvw_0)^T\|_2\\
&\le&\|D \F(\rvw)\|_2\|D \F(\rvw)-D \F(\rvw_0)\|_2 + \|D \F(\rvw)-D \F(\rvw_0)\|_2\|D \F(\rvw_0)\|_2\\
&\le&L_\F\cdot\|D \F(\rvw)-D \F(\rvw_0)\|_F + \|D \F(\rvw)-D \F(\rvw_0)\|_F\cdot L_\F\\
&\le&2L_\F\cdot \frac{\lambda_0-\mu}{2L_\F}\\
&=&\lambda_0-\mu.
\end{eqnarray*}
In the second inequality above, we used the $L_\F$-Lipschitz continuity of $\F$ and the fact that $\|A\|_2 \le \|A\|_F$ for a matrix $A$.

By triangular inequality, we have, at any point $\rvw\in B(\rvw_0,R)$,
\begin{equation}
    \lambda_{min}(K(\rvw)) \ge \lambda_{min}(K(\rvw_0))-\|K(\rvw)-K(\rvw_0)\|_2 \ge \mu.
\end{equation}
Hence, the tangent kernel is $\mu$-uniformly conditioned in the ball $B(\rvw_0,R)$.

By Theorem~\ref{thm:wellconditionandpl}, we immediately have that the square loss $\l(\rvw)$ satisfies $\mu$-PL$^*$ condition in the ball $B(\rvw_0,R)$.
\end{proof}

Below in Section~\ref{sec:wide_nn_pl*}, we will see that wide neural networks with linear output layer have small Hessian norm (Theorem~\ref{thm:ntk}). An illustration  for a class of large models is given in Appendix~\ref{appsec:small_hessian}.

\paragraph{Conditioning of transformed systems.}

We now discuss why the conditioning of a system $\F(\rvw)=\rvy$ is preserved under a transformations of the domain or range of $\F$, as long as the original system is well-conditioned and the transformation has a bounded inverse.

\begin{remark}
Note that the even if the original system had a small Hessian norm, there is no such guarantee for the transformed system.
\end{remark}

Consider a transformation $\Phi: \mathbb{R}^n \to \mathbb{R}^n$ that, composed with $\F$, results in a new transformed system $\Phi\circ \F (\rvw)=\rvy$. 
Put 
\begin{equation}\label{eq:jacobian}
     \rho := \inf_{\rvw\in B(\rvw_0,R)} \frac{1}{\|J_{\Phi}^{-1}(\rvw)\|_2},
\end{equation}
where $J_{\Phi}(\rvw):= J_{\Phi}(\F(\rvw))$ is the Jacobian of $\Phi$ evaluated at $\F(w)$.
We will assume that $\rho>0$.

\begin{thm}\label{thm:ext_pl}
If a system $\F$ is $\mu$-uniformly conditioned in a ball $B(\rvw_0,R)$ with $R>0$,  then the transformed system $\Phi\circ \F(\rvw)$ is $\mu\rho^2$-uniformly conditioned in $B(\rvw_0,R)$.
Hence, the square loss function $\frac{1}{2}\|\Phi\circ \F(\rvw)-\rvy\|^2$ satisfies $\mu\rho^2$-PL$^*$ condition in $B(\rvw_0,R)$.
\end{thm}
\begin{proof}
First, note that,
\begin{align*}
    K_{\Phi \circ \F}(\rvw) = \nabla (\Phi \circ \F)(\rvw)\nabla (\Phi \circ \F)^T(\rvw)
  =J_{\Phi}(\rvw)K_\F(\rvw;X) J_{\Phi}(\rvw)^T.
\end{align*}

\noindent Hence, if $\F(\rvw)$ is $\mu$-uniformly conditioned in $B(\rvw_0,R)$, i.e. $\lambda_{min}(K_\F(\rvw)) \geq \mu$, we have for any $\rvv \in \mathbb{R}^n$ with $\| \rvv\| =1$,
\begin{align*}
    \rvv^T K_{\Phi \circ \F}(\rvw)\rvv &= (J_{\Phi}(\rvw)^T \rvv)^T K_{\F}(\rvw)(J_{\Phi}(\rvw)^T \rvv)\\
    &\geq \lambda_{\min}(K_\F(\rvw)) \| J_{\Phi}(\rvw)^T \rvv\|^2   \\
    &\geq \lambda_{\min}(K_\F(\rvw)) /\|J_{\Phi}^{-1}(\rvw)\|_2^2 \ge \mu\rho^2.
\end{align*}
Applying Theorem~\ref{thm:wellconditionandpl}, we immediately obtain that $\frac{1}{2}\|\Phi\circ \F(\rvw)-\rvy\|^2$ satisfies $\mu\rho^2$-PL$^*$ condition in $B(\rvw_0,R)$.
\end{proof}

\begin{remark}\label{remark:transform}
A result analogous to Theorem~\ref{thm:ext_pl} is easy to obtain for a system with transformed input $\F \circ \Psi(\rvw) = \rvy$.
Assume the transformation map $\Psi: \mathbb{R}^m \rightarrow \mathbb{R}^m$ that applies on the input of the system $\rvw$ satisfies
\begin{equation}
    \rho := \underset{\rvw \in B(\rvw_0,R)}{\inf}\frac{1}{\| J_\Psi^{-1} (\rvw)\|_2} > 0.
\end{equation}
If $\F$ is $\mu$-uniformly conditioned with respect to $\Psi(\rvw)$ in $B(\rvw_0,R)$, then an analysis similar to Theorem~\ref{thm:ext_pl} shows that $\F \circ \Psi$ is also $\mu\rho^2$-uniformly conditioned in $B(\rvw_0,R)$.
\end{remark}


\paragraph{Conditioning of composition  models.}
Although composing different large models often leads to non-constant tangent kernels, the corresponding tangent kernels can also be uniformly conditioned, under certain conditions. Consider the composition of two models $h := g\circ f$, where $f: \mathbb{R}^d \to \mathbb{R}^{d'}$ and $g: \mathbb{R}^{d'}\to \mathbb{R}^{d''}$. Denote $\rvw_f$ and $\rvw_g$ as the parameters of model $f$ and $g$ respectively. Then, the parameters of the composition model $h$ are $\rvw:= (\rvw_g,\rvw_f)$. 
Examples of the composition models include "bottleneck" neural networks, where the modules below or above the bottleneck layer can be considered as the composing (sub-)models.

Let's denote the tangent kernel matrices of models $g$ and $f$ by $K_g(\rvw_g;\mathcal{Z})$ and $K_f(\rvw_f;\mathcal{X})$ respectively, where the second arguments, $\mathcal{Z}$ and $\mathcal{X}$, are the datasets that the tangent kernel matrices are evaluated on. Given a dataset $\mathcal{D} = \{(\rvx_i, y_i)\}_{i=1}^n$, denote $f(\mathcal{D})$ as $\{(f(\rvx_i), y_i)\}_{i=1}^n$.
\begin{prop}\label{prop:composition}
Consider the composition model $h=g\circ f$ with parameters $\rvw = (\rvw_g, \rvw_f)$. Given a dataset $\mathcal{D}$, the tangent kernel matrix of $h$ takes the form:
$$K_h(\rvw; \mathcal{D}) = K_g(\rvw_g;f(\mathcal{D})) + J_{g}(f(\mathcal{D}))K_f(\rvw_f;\mathcal{D})J_{g}(f(\mathcal{D}))^T,$$
where $J_{g}(f(\mathcal{D})) \in \mathbb{R}^{d'' \times d'}$ is the Jacobian of $g$ w.r.t. $f$ evaluated on $f(\mathcal{D})$.
\end{prop}

From the above proposition, we see that the tangent kernel of the composition model $h$ can be decomposed into the sum of two positive semi-definite matrices, hence the minimum eigenvalue of $K_h$ can be lower bounded by  
\begin{equation}\label{eq:composition}
    \lambda_{\min} (K_{h}(\rvw)) \geq  \lambda_{\min}(K_g(\rvw_g;f(\mathcal{D}))).
\end{equation}

We note that $g$ takes the outputs of $f$ as inputs which depends on $\rvw_f$, while for the model $f$ the inputs are fixed. 
Hence, if the model $g$ is uniformly conditioned at all the inputs provided by the model $f$, we can expect the uniform conditioning of the composition model $h$.

We provide a simple illustrative example for a ``bottleneck'' neural network in Appendix~\ref{secapp:compostion}.

\subsection{Wide neural networks satisfy PL$^*$ condition}\label{sec:wide_nn_pl*}
In this subsection, we show that wide neural networks satisfy the PL$^*$ condition, using the techniques we developed in the last subsection.

A $L$-layer (feedforward) neural network $f(\rmW;\rvx)$, with parameters $\rmW$ and input $\rvx$, is defined as follow:
\begin{align}\label{eq:generalnn}
 &\alpha^{(0)} = \rvx, \nonumber\\
 &\alpha^{(l)}= \sigma_{l}\left(\frac{1}{\sqrt{m_{l-1}}}W^{(l)}\alpha^{(l-1)}\right), \  \ \forall l = 1,2,\cdots, L+1,\nonumber\\
 &f(\rmW;\rvx) = \alpha^{(L+1)}.
\end{align} 

Here, $m_l$ is the width (i.e., number of neurons) of $l$-th layer, $\alpha^{(l)}\in\mathbb{R}^{m_l}$ denotes the vector of $l$-th hidden layer neurons, $\rmW := \{W^{(1)},W^{(2)},\ldots ,W^{(L)},W^{(L+1)}\}$ denotes the collection of the parameters (or weights) $W^{(l)}\in \mathbb{R}^{m_{l}\times m_{l-1}}$ of each layer, and $\sigma_{l}$ is the activation function of $l$-th layer, e.g., $sigmoid$, $tanh$, linear activation. We also denote the width of the neural network as $m:=\min_{l\in[L]}m_l$, i.e., the minimal width of the hidden layers. In the following analysis, we assume that the activation functions $\sigma_l$ are  twice differentiable. Although this assumption excludes ReLU, but we believe the same results also apply when the hidden layer activation functions are ReLU.

\begin{remark}
The above definition of neural networks does not include convolutional (CNN) and residual (ResNet) neural networks. In Appendix~\ref{appsec:cnn_resnet}, we show that both CNN and ResNet also satisfy the PL$^*$ condition. Please see the definitions and analysis there.
\end{remark} 

We study the loss landscape of wide neural networks in regions around randomly chosen points in parameter space. Specifically, we consider the ball $B(\rmW_0,R)$, which has a fixed radius $R>0$ (we will see later, in Section~\ref{sec:pl-conv}, that $R$ can be chosen to cover the whole optimization path) and is around a random parameter point $\rmW_0$, i.e., $W_0^{(l)}\sim \mathcal{N}(0,I_{m_l\times m_{l-1}})$ for $l\in[L+1]$. Note that such a random parameter point $\rmW_0$ is a common choice to initialize a neural network. {Importantly,  the tangent kernel matrix at $\rmW_0$ is generally strictly positive definite, i.e., $\lambda_{min}(K(\rmW_0))>0$. Indeed, this is proven for infinitely wide networks
as long as the training data is not degenerate (see Theorem 3.1 of \cite{du2018gradientshallow} and Proposition F.1 and F.2 of \cite{du2018gradientdeep}). As for  finite width networks,  with high probability w.r.t. the initialization randomness, its tangent kernel $K(\rmW_0)$ is close to that of the infinite network and the minimum eigenvalue $\lambda_{min}(K(\rmW_0)) = O(1)$.}

Using the techniques in Section~\ref{subsec:uc_ball}, the following theorem shows that neural networks with sufficient width satisfies the PL$^*$ condition in a ball of any fixed radius around $\rmW_0$, as long as the tangent kernel $K(\rmW_0)$ is strictly positive definite.

\begin{thm}[Wide neural networks satisfy PL$^{*}$ condition]\label{thm:deepnnpl}
Consider the neural network $f(\rmW;\rvx)$ in Eq.(\ref{eq:generalnn}), and a random parameter setting $\rmW_0$ such that $W_0^{(l)}\sim \mathcal{N}(0,I_{m_l\times m_{l-1}})$ for $l\in[L+1]$. Suppose that the last layer activation $\sigma_{L+1}$ satisfies $|\sigma'_{L+1}(z)| \ge \rho > 0$ and that  $\lambda_0:=\lambda_{min}(K(\rmW_0)) >0$.  For any $\mu \in (0,\lambda_0\rho^2)$, if the width of the network 
\begin{equation}\label{eq:width}
    m = \tilde{\Omega}\left(\frac{nR^{6L+2}}{(\lambda_0-\mu\rho^{-2})^2}\right),
\end{equation}
then $\mu$-PL$^*$ condition holds the square loss function in the ball $B(\rvw_0,R)$.
\end{thm}

\begin{remark}
In fact, it is not necessary to require $|\sigma_{L+1}'(z)|$ to be greater than $\rho$ for all $z$. The theorem still holds as long as $|\sigma_{L+1}'(z)|>\rho$ is true for all $z$ actually achieved by the output neuron before activation. 
\end{remark}
This theorem tells that while the loss landscape of wide neural networks is nowhere convex (as seen in Section~\ref{sec:nonconvex}), it can still can be described by the PL$^*$ condition at most points, in line with our general discussion. 

\paragraph{Proof of Theorem~\ref{thm:deepnnpl}.}

We divide the proof into two distinct steps based on representing an arbitrary neural network as a composition of network with a linear output layer and an output non-linearity $\sigma_{L+1}(\cdot)$. 
In Step 1 we prove the PL$^*$ condition for the case of a network with a linear output layer (i.e., $\sigma_{L+1}(z) \equiv z$). The argument relies on the fact that wide neural networks with linear output layer have small Hessian norm in a ball around initialization. 
In Step 2  for general networks we observe that an arbitrary neural network is simply a neural network with a linear output layer from Step 1 with output transformed by applying  $\sigma_{L+1}(z)$ coordinate-wise. We obtain the result by combining Theorem~\ref{thm:ext_pl} with Step 1.


\paragraph{Step 1. Linear output layer: $\sigma_{L+1}(z)\equiv z$.} 
In this case, $\rho=1$ and the output layer of the network has a linear form, i.e., a linear combination of the units from the last hidden layer.

As was shown in~\cite{liu2020linearity}, for this type of networks with sufficient width, the model Hessian matrix have arbitrarily small spectral norm (a {\it transition to linearity}). This is formulated in the following theorem:


\begin{thm}[Theorem 3.2 of \cite{liu2020linearity}: transition to linearity] \label{cor:different_width}
Consider a neural network $f(\rmW;\rvx)$ of the form Eq.(\ref{eq:generalnn}). Let $m$ be the minimum of the hidden layer widths, i.e., $m= \min_{l\in[L]} m_l$. Given any fixed $R>0$, and any  $\rmW \in B(\rmW_0,R):= \{\rmW: \|\rmW - \rmW_0\| \le R\}$,   with high probability over the initialization, the Hessian spectral norm satisfies the following:
\begin{equation}\label{eq:hessian_bound_nn}
    \|H_f(\rmW)\| = \tilde{O}\left(R^{3L}/{\sqrt{m}}\right). 
\end{equation}
\end{thm}
In Eq.(\ref{eq:hessian_bound_nn}), we explicitly write out the dependence of Hessian norm on the radius $R$, according to the proof in \cite{liu2020linearity}.

Directly plugging Eq.(\ref{eq:hessian_bound_nn}) into the condition of Theorem~\ref{thm:ntk}, letting $\epsilon = \lambda_0-\mu$ and noticing that $\rho=1$, we directly have the expression for the width $m$.

\paragraph{Step 2. General networks: $\sigma_{L+1}(\cdot)$ is non-linear.}


Wide neural networks with non-linear output layer generally do not exhibit transition to linearity  or near-constant tangent kernel, as was shown~\cite{liu2020linearity}. Despite that, these wide networks still satisfy the PL$^*$ condition in the ball $B(\rmW_0,R)$.
Observe that this type of network can be viewed as a composition of a non-linear transformation function $\sigma_{L+1}$ with a network $\tilde{f}$ which has a linear output layer:
\begin{equation}
    {f}(\rmW;\rvx) = \sigma_{L+1}(\tilde{f}(\rmW;\rvx)).
\end{equation}
By the same argument as in Step 1, we see that $\tilde{f}$, with the width as in Eq.(\ref{eq:width}), is $\frac{\mu}{\rho^{2}}$-uniformly conditioned.

Now we apply our analysis for transformed systems in Theorem~\ref{thm:ext_pl}. In this case, the transformation map $\Phi$ becomes a coordinate-wise transformation of the output given by
\begin{equation}
    \Phi(\cdot) = \diag{(\underbrace{\sigma_{L+1}(\cdot),\sigma_{L+1}(\cdot), \cdots, \sigma_{L+1}(\cdot)}_{n})},
\end{equation}
and the norm of the inverse Jacobian matrix, $\|J_{\Phi}^{-1}(\rvw)\|_2$ is 

\begin{equation}
    \|J_{\Phi}^{-1}(\rvw)\|_2 = \frac{1}{\min_{i\in[n]} \left|\sigma_{L+1}'(\tilde{f}(\rvw;\rvx_i))\right|} \le \rho^{-1}.
\end{equation}
Hence, the $\frac{\mu}{\rho^{2}}$-uniform conditioning of $\tilde{f}$ immediately implies $\mu$-uniform conditioning of $f$, as desired.

\section{PL$^*$ condition in a ball guarantees 
existence of solutions and fast convergence of (S)GD }\label{sec:pl-conv}
In this section, we show that fast convergence of gradient descent methods is guaranteed by the PL$^*$ condition in a ball with appropriate size. 
We assume the system $\F(\rvw)$ is $L_\F$-Lipschitz continuous and $\beta_\F$-smooth on the local region $\mathcal{S}$ that we considered. In what follows $\mathcal{S}$ will typically be a Euclidean ball $B(\rvw_0,R)$, with an appropriate radius $R$, chosen to cover the optimization path of GD or SGD. 

First, we define the (non-linear) condition number, as follows:
\begin{defi}[Condition number]\label{defi:oconditionnumber}
Consider a system in Eq.(\ref{eq:problem}), with a loss function $\l(\rvw)=\l(\F(\rvw),\rvy)$ and a set $\mathcal{S}\subset \mathbb{R}^m$. If the loss $\l(\rvw)$ is $\mu$-PL$^*$ conditioned on $\mathcal{S}$,  define the condition number $\kappa_{\l,\F}(\mathcal{S})$:
\begin{align}\label{eq:def_cond}
    \kappa_{\l,\F}(\mathcal{S}) := \frac{\sup_{\rvw\in \mathcal{S}}  \lambda_{max}(H_\l(\rvw))}{\mu}, 
\end{align}
where $H_\l(\rvw)$ is the Hessian matrix of the loss function. 
The condition number for {\it the square loss} (used throughout the paper) will be written as simply  $\kappa_\F(\mathcal{S})$, omitting the subscript $\l$.
\end{defi}
\begin{remark}
In the special case of a linear system $\F(\rvw) = A\rvw$ with square loss $\frac{1}{2}\|A\rvw-\rvy\|^2$, both the Hessian $ H_\l=A^T A$ and the tangent kernel $K(\rvw)=A\,A^T$ are constant matrices. 
As $AA^T$ and $A^TA$ have the same set of non-zero eigenvalues, the largest eigenvalue $\lambda_{max}(H_\l)$ is equal to $\lambda_{max}(K)$. In this case, the condition number $\kappa_{\F}(\mathcal{S})$ reduces to the standard condition number of the tangent kernel $K$,
\begin{equation}
    \kappa_{\F}(\mathcal{S}) = \frac{\lambda_{max}(K)}{\lambda_{min}(K)}.
\end{equation}
\end{remark}


{Since $\F$ is $L_\F$-Lipschitz continuous and $\beta_\F$ smooth by assumption, we directly get the following by substituting the definition of the square loss function into Eq.~(\ref{eq:def_cond}).}
\begin{prop}\label{prop:cond_num} 
For the square  loss function $\l$,  the  condition number is upper bounded by:
\begin{equation}
    \kappa_\F(\mathcal{S}) \le \frac{L_\F^2 + \beta_\F\cdot \sup_{\rvw\in\mathcal{S}}\|\F(\rvw)-\rvy\|}{\mu}.
\end{equation}
\end{prop}
\begin{remark}
It is easy to see that the usual condition number $\kappa(\rvw) = \lambda_{max}(K(\rvw))/\lambda_{min}(K(\rvw))$ of the tangent kernel $K(\rvw)$, is upper bounded  by  $\kappa_\F(\mathcal{S})$.
\end{remark}


Now, we are ready to present the optimization theory based on the PL$^*$ condition. First, let the arbitrary set $\mathcal{S}$ to be a Euclidean ball $B(\rvw_0,R)$ around the initialization $\rvw_0$ of gradient descent methods, with a reasonably large but finite radius $R$. The following theorem shows that satisfaction of the PL$^*$ condition on $B(\rvw_0,R)$ implies the existence of at least one global solution of the system in the same ball $B(\rvw_0,R)$. Moreover, following the original argument from~\cite{polyak1963gradient}, the PL$^*$ condition also implies fast convergence of gradient descent to a global solution $w^*$ in the ball $B(\rvw_0,R)$.

\begin{thm}[Local PL$^*$ condition $\Rightarrow$ existence of a solution + fast convergence]\label{thm:plstar}
Suppose the system $\F$ is $L_\F$-Lipschitz continuous and $\beta_\F$-smooth. If the square loss $\l(\rvw)$ satisfies the $\mu$-PL$^*$ condition in the ball $B(\rvw_0,R) := \{\rvw\in \mathbb{R}^m : \|\rvw-\rvw_0\| \le R\}$ with $R = \frac{2L_\F\|\F(\rvw_0)-\rvy\|}{\mu}$. Then we have the following:

\noindent (a) Existence of a solution: There exists a solution (global minimizer of $\l$) $\rvw^* \in B(\rvw_0,R)$, such that $\F(\rvw^{*})=\rvy$.

\noindent (b) Convergence of GD: Gradient descent with a step size $\eta \le 1/(L_\F^2 + \beta_\F \|\F(\rvw_0)-\rvy\|)$ converges to a global solution in  $B(\rvw_0,R)$, with an exponential (a.k.a. linear) convergence rate: 
\begin{equation}
    \l(\rvw_t) \le \left(1-{\kappa_{\F}^{-1}(B(\rvw_0,R))}\right)^t\l(\rvw_0).
\end{equation}
where the condition number $\kappa_{\F}(B(\rvw_0,R)) = \frac{1}{\eta\mu}$.
\end{thm}

The proof of the theorem is deferred to Appendix \ref{secapp:plstar}.

It is interesting to note that the radius $R$ of the ball $B(\rvw_0,R)$ takes a finite value, which means that the optimization path $\{\rvw_t\}_{t=0}^\infty\subset B(\rvw_0,R)$ stretches at most a finite length and the optimization happens only at a finitely local region around the initialization. Hence, the conditioning of the tangent kernel and the satisfaction of the PL$^*$ condition outside of this ball are irrelevant to this optimization and are not required. 

Indeed, this radius $R$ has to be of order $\Omega(1)$. From the $L_\F$-Lipschitz continuity of  $\F(\rvw)$, it follows that there is no solution within the distance $\|\F(\rvw_0)-\rvy\|/L_\F$ from the initialization point. 
Any solution must have a Euclidean distance away from $\rvw_0$ at least $R_{min} = \|\F(\rvw_0)-\rvy\|/L_\F$, which is finite.
This means that the parameter update $\Delta \rvw = \rvw^*-\rvw_0$ must not be too small in terms of Euclidean distance. However, due to the large population $m$ of the model parameters, each individual parameter $w_i$ may take a small change during the gradient descent training, i.e., $|w_i-w_{0,i}| = O(1/\sqrt{m})$. Indeed, this is what happening for wide neural networks~\cite{jacot2018neural,liu2020linearity}.


Below, we make an extension of the above theory: from (deterministic) gradient descent to stochastic gradient descent (SGD).

\paragraph{Convergence of SGD.}
In most practical machine learning settings, including typical problems of supervised learning, the loss function $\l(\rvw)$ has the form 
$$
\l(\rvw) = \sum_{i=1}^n \ell_i(\rvw).
$$

For example, for the square loss $\l(\rvw) = \sum_{i=1}^{n} \ell_i(\rvw)$,   where $\ell_i(\rvw) = \frac{1}{2} ( \F_i(\rvw) - y_i)^2$. Here the loss $\ell_i$ corresponds simply to the loss for $i$th equation.
Mini-batch SGD  updates the parameter $\rvw$, according to the gradient of $s$ individual loss functions $\ell_i(\rvw)$ at a time:
\begin{equation*}
    \rvw_{t+1} = \rvw_t - \eta\sum_{i\in \mathcal{S}\subset[n]}\nabla\ell_i(\rvw_t), \forall t\in \mathbb{N}.
\end{equation*}
\noindent
We will assume that each element of the set $S$ is chosen uniformly at random at every iteration. 


We now show that the PL$^*$ condition on $\l$ also implies exponential convergence of SGD within a ball, an SGD analogue of Theorem~\ref{thm:plstar}.   
Our result can be considered as a local version of  Theorem~1 in~\cite{bassily2018exponential} which showed  exponential convergence of SGD,  assuming PL condition holds in the entire parameter space. See also~\cite{oymak2019overparameterized} for a related result. 
\begin{thm}\label{thm:sgd_plstar}
Assume each  $\ell_i(\rvw)$  is $\beta$-smooth and $\l(\rvw)$ satisfies the $\mu$-$PL^*$ condition in the ball $B(\rvw_0,R)$ with $R=\frac{2n\sqrt{2\beta\l(\rvw_0)}}{\mu \delta}$ where $\delta>0$. Then, with probability $1-\delta$, SGD with mini-batch size $s\in \mathbb{N}$ and step size $\eta\le \frac{n\mu}{n\beta(n^2\beta+\mu(s-1))}$ converges to a global solution in the ball $B(\rvw_0,R)$, with an exponential convergence rate:
\begin{align}
    \mathbb{E}[\l(\rvw_{t})]  \leq \left(1-\frac{\mu s \eta}{n}\right)^t \l(\rvw_{0}).
\end{align}
\end{thm}

\noindent The proof is deferred to Appendix \ref{prf:sgd_plstar}.

\subsection{Convergence for wide neural networks.}
Using the optimization theory developed above, we can now show convergence of (S)GD for sufficiently wide neural networks.
As we have seen in Section~\ref{sec:wide_nn_pl*},  the  loss landscapes of wide neural networks $f(\rmW;\rvx)$ defined in Eq.(\ref{eq:generalnn})  are $\mu$-PL$^*$ condition in a ball $B(\rmW_0,R)$ of an arbitrary radius $R$. We will now show convergence of GD for these models. 


As before (Section~\ref{sec:wide_nn_pl*}) write $f(\rmW;\rvx) = \sigma_{L+1}(\tilde{f}(\rmW;\rvx))$ where $\tilde{f}(\rmW;\rvx)$ is a neural network with a linear output layer and we denoted $\F(\rmW) = (f(\rmW;\rvx_1), \ldots, f(\rmW; \rvx_n))$, $ \tilde{\F}(\rmW) = (\tilde{f}(\rmW;\rvx_1), \ldots, \tilde{f}(\rmW; \rvx_n))$ and $\rvy = (y_1,\ldots, y_n)$. We use tilde, e.g., $\tilde{O}(\cdot)$ to suppress logarithmic terms in Big-O notation.

We further assume $\sigma_{L+1}(\cdot)$ is $L_\sigma$-Lipschitz continuous and $\beta_\sigma$-smooth.


\begin{thm}\label{cor:deepnn}
Consider the neural network $f(\rmW; \rvx)$ and its random initialization $\rmW_0$ under the same condition as in Theorem~\ref{thm:deepnnpl}. If the network width satisfies $m = \tilde{\Omega}\left(\frac{n}{\mu^{6L+2}(\lambda_0-\mu\rho^{-2})^2}\right)$,
then, 
with an appropriate step size, gradient descent converges to a global minimizer in the ball $B(\rmW_0,R)$, where 
$R = O(1/\mu)$, 
with an exponential convergence rate:
\begin{equation}
    \l(\rmW_t) \le (1-\eta\mu)^t \l(\rmW_0).
\end{equation}
\end{thm}
\begin{proof}
The result is obtained by combining Theorem~\ref{thm:deepnnpl} and \ref{thm:plstar}, after setting the ball radius $R= 2L_\F\|\F(\rmW_0)-\rvy\|/\mu$ and choosing the step size $0<\eta < \frac{1}{L_\F^2 + \beta_\F\|\F(\rmW_0)-\rvy\|}$.

It remains to verify  that all of  the quantities $L_\F$, $\beta_\F$ and $\|\F(\rmW_0)-\rvy\|$ are of order $O(1)$ with respect to the network width $m$.
First note that, with the random initialization of $\rmW_0$ as in Theorem~\ref{thm:deepnnpl}, it is shown  in \cite{jacot2018neural} that with high probability,  $f(\rmW_0;\rvx) = O(1)$ when the width is sufficiently large under mild assumption on the non-linearity functions $\sigma_l(\cdot)$, $l=1,...,L+1$. Hence $\|\F(\rmW_0)-\rvy\| = O(1)$.

 Using the definition of $L_\F$, we have
\begin{align*}
    L_\F  &=\underset{\rmW \in B(\rmW_0,R)}{\sup} \|D\F(\rmW)\| \\
    &\leq L_\sigma\left(\|D\tilde{\F}(\rmW_0)\| + R
    \sqrt{n}\cdot\underset{\rmW \in B(\rmW_0,R)}{\sup} \|\rmH_{\tilde{\F}}(\rmW)\| \right)\\
    &= \sqrt{\|K_{\tilde{\F}}(\rmW_0)\|}L_\sigma + L_\sigma R\sqrt{n} \cdot \tilde{O}\left(\frac{1}{\sqrt{m}}\right) = O(1),
\end{align*}
where the last inequality follows from Theorem~\ref{cor:different_width}, and $\|K_{\tilde{\F}}(\rmW_0)\| = O(1)$ with high probability over the random initialization.

Finally, $\beta_\F$ is bounded as follows:
\begin{align*}
    \beta_\F &= \underset{\rmW \in B(\rmW_0,R)}{\sup} \|\rmH_\F(\rmW)\| \\
    &=   \underset{\rmW \in B(\rmW_0,R)}{\sup} \|H_{f_k}(\rmW)\| \\
    &=  \underset{\rmW \in B(\rmW_0,R)}{\sup} \left\| \sigma_{L+1}''\left(\tilde{f}_k(\rmW)\right) D\tilde{f}_k(\rmW) D\tilde{f}_k(\rmW)^T  + \sigma_{L+1}'\left(\tilde{f}_k(\rmW)\right) H_{\tilde{f}_k}(\rmW)\right\|\\ 
    &\leq  \beta_\sigma \left(\underset{\rmW \in B(\rmW_0,R)}{\sup} \|D\tilde{f}_k(\rmW)\|\right)^2 + L_\sigma \underset{\rmW \in B(\rmW_0,R)}{\sup} \|H_{\tilde{f}_k}(\rmW)\| \\
    &\leq \beta_\sigma \cdot O(1) + L_\sigma \cdot \tilde{O}\left(\frac{1}{\sqrt{m}}\right) = O(1),
\end{align*}
where $k = \underset{i\in[n]}{\mathrm{argmax }}\left(\underset{\rmW \in B(\rmW_0,R)}{\sup}\|H_{f_i}(\rmW)\|\right)$.
\end{proof}
\begin{remark}
Using the same argument, a result similar to Theorem~\ref{cor:deepnn} but with a different convergence rate,
\begin{equation}
    \l(\rmW_t) \le (1-\eta s \mu)^t \l(\rmW_0).
\end{equation}
and with difference constants, can be obtained for SGD by applying Theorem~\ref{thm:sgd_plstar}.
\end{remark}



Note that (near) constancy of the  tangent kernel is not a necessary condition  for exponential convergence of gradient descent or (S)GD.

\section{Relaxation to PL$^*_\epsilon$ condition}\label{subsec:pl-epsilon}
In certain situations, for example mildly under-parameterized cases, the PL$^*$ condition may not hold exactly, since an exact solution for system $\F(\rvw)=\rvw$ may not exist. Fortunately, in practice, we do not need to run algorithms until exact convergence. Most of the time, 
{\it early stopping} is employed, i.e. we 
stop the algorithm once it achieves a certain small loss $\epsilon>0$. To account for that  case, we define PL$^*_\epsilon$ condition, a relaxed variant of PL$^*$ condition, which still implies a fast convergence of the gradient-based algorithms up to loss $\epsilon$.

\begin{defi}[PL$^*_{\epsilon}$ condition]\label{defi:pl_epsilon}
Given a set $\mathcal{S}\subset \mathbb{R}^m$ and $\epsilon>0$, define the set $\mathcal{S}_{\epsilon} := \{\rvw \in \mathcal{S} : \mathcal{L}(\rvw) \ge \epsilon\}$. A loss function $\mathcal{L}(\rvw)$ is $\mu$-PL$^*_{\epsilon}$ on $\mathcal{S}$, if the following holds:
\begin{equation}
    \frac{1}{2}\|\nabla \l(\rvw)\|^2 \ge \mu \l(\rvw), \ \forall \rvw\in \mathcal{S}_\epsilon.
\end{equation}
\end{defi}

Intuitively, the PL$^*_\epsilon$ condition is the same as PL$^*$ condition, except that the loss landscape can be arbitrary wherever the loss is less than $\epsilon$. This is illustrated in Figure~\ref{fig:pl-epsilon}.

  
  \begin{wrapfigure}{R}{0.4\textwidth}
  \vspace{-26pt}
\begin{center}
  \includegraphics[width=0.4\textwidth]{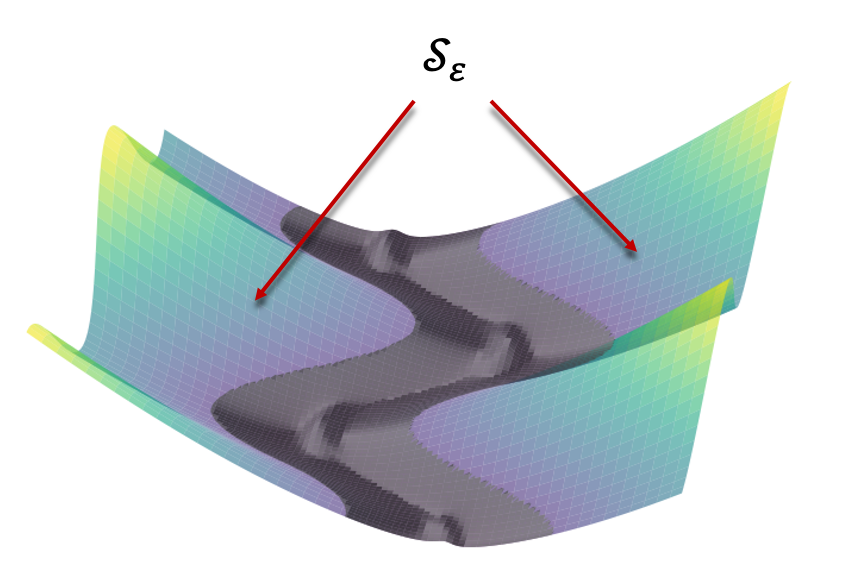} 
\end{center}
\caption{The loss landscape of under-parameterized systems. In the set $\mathcal{S}_\epsilon$, where the loss is larger than $\epsilon$, PL$^*$ still holds. Beyond that, the loss landscape can be arbitrary, which is the grey area, and PL$^*$ doesn't hold any more.}
\vspace{-20pt} \label{fig:pl-epsilon}
\end{wrapfigure}

Below, following a similar argument above, we show that a {\it local} PL$^*_\epsilon$ condition guarantees fast convergence to an approximation of a global minimizer. Basically, the gradient descent terminates when the loss is less than $\epsilon$.


\begin{thm}[Local PL$^*_\epsilon$ condition $\Rightarrow$fast convergence]\label{thm:plstarep}
Assume the loss function $\l(\rvw)$ (not necessarily the square loss) is $\beta$-smooth and  satisfies the $\mu$-PL$^*_\epsilon$ condition in the ball $B(\rvw_0,R) := \{\rvw\in \mathbb{R}^m : \|\rvw-\rvw_0\| \le R\}$ with $R = \frac{2\sqrt{2\beta \l(\rvw_0)}}{\mu}$. We have the following:

 \noindent (a) Existence of a point which makes the loss less than $\epsilon$: There exists a point  $\rvw^* \in B(\rvw_0,R)$, such that $\l(\rvw^*)< \epsilon$.

\noindent (b) Convergence of GD : Gradient descent with the step size $\eta = 1/\sup_{\rvw\in B(\rvw_0,R)}\|H_\l(\rvw)\|_2$ after $T = \Omega(\log(1/\epsilon))$  iterations satisfies $\l(\rvw_T) < \epsilon$ in the ball $B(\rvw_0,R)$, with an exponential (also known as linear) convergence rate: 
\begin{equation}
    \l(\rvw_t) \le \left(1-{\kappa_{\l,F}^{-1}(B(\rvw_0,R))}\right)^t\l(\rvw_0), \ \textrm{for all } \ t \leq T,
\end{equation}
where the condition number $\kappa_{\l,F}(B(\rvw_0,R)) = \frac{1}{\eta\mu}$.

 \end{thm}

\begin{proof}
If $\l(\rvw_0) < \epsilon$, we let $\rvw^* = \rvw_0$ and we are done. 

Suppose $\l(\rvw_0) \geq \epsilon$. Following the similar analysis to the  proof of Theorem \ref{thm:plstar}, as long as $\l(\rvw_t) \geq \epsilon$ for $t\geq 0$, we have
\begin{align}
    \l(\rvw_t) \leq (1-\eta \mu)^t \l(\rvw_0).
\end{align}

Hence there must exist a minimum $T >0$ such that 
\begin{align}
    \l(\rvw_T) < \epsilon.
\end{align}
It's not hard to see that $\epsilon \leq (1-\eta\mu)^T\l(\rvw_0)$, from which we get $T = \Omega(\log(1/\epsilon)).$

And obviously, $\rvw_0,\rvw_1,...,\rvw_T$ are also in the ball $B(\rvw_0,R)$ with $R = \frac{2\sqrt{2\beta\l(\rvw_0)}}{\mu}$.
\end{proof}

Following similar analysis, when PL$^*_\epsilon$ condition holds in the ball, SGD converges exponentially to an approximation of the global solution.
\begin{thm}\label{thm:sgd_plstar_epsilon}
Assume each  $\ell_i(\rvw)$  is $\gamma$-smooth and $\l(\rvw)$ satisfies the $\mu$-$PL^*_{\epsilon}$ condition in the ball $B(\rvw_0,R)$ with $R=\frac{2n\sqrt{2\gamma}\sqrt{\l(\rvw_0)}}{\mu \delta}$ where $\delta>0$. Then, with probability $1-\delta$, SGD with mini-batch size $s\in \mathbb{N}$ and step size $\eta^*(s) =\frac{n\mu}{n\gamma(n^2\gamma+\mu(s-1))}$ after at most $T = \Omega(\log(1/\epsilon))$ iterations satisfies $\min\left(\mathbb{E}[\l(\rvw_T)],\l(\rvw_T)\right) < \epsilon$ in the ball $B(\rvw_0,R)$, with an exponential convergence rate:
\begin{align}
    \mathbb{E}[\l(\rvw_{t})]  \leq \left(1-\frac{\mu s \eta^*(s)}{n}\right)^t \l(\rvw_{0}),\ t \leq T.
\end{align}
\end{thm}
\begin{proof}
If $\l(\rvw_0) < \epsilon$, we let $T = 0$ and we are done.

Suppose $\l(\rvw_0) \geq \epsilon$. By the similar analysis in the proof of Theorem~\ref{thm:sgd_plstar}, with the $\mu$-PL$^*_\epsilon$ condition, for any mini-batch size $s$, the mini-batch SGD with step size $\eta^*(s) := \frac{n\mu}{n\lambda(n^2\lambda+\mu(s-1))}$ has an exponential convergence rate:
\begin{align}
    \mathbb{E}[\l(\rvw_{t})]  \leq \left(1-\frac{\mu s \eta^*(s)}{n}\right)^t \l(\rvw_{0}),
\end{align}
for all $t$ where $\l(\rvw_t) \geq \epsilon$.

Hence there must exist a minimum $T >0$ such that either $L(\rvw_T) < \epsilon $ or  $\mathbb{E}[\l(\rvw_{T})] < \epsilon$. Supposing $\mathbb{E}[\l(\rvw_{T})] < \epsilon$ at time step $T$, we have $\epsilon \leq \left(1-\frac{\mu s \eta^*(s)}{n}\right)^T \l(\rvw_{0})$, therefore $T = \Omega(\log(1/\epsilon)).$

And it's easy to check that with probability $1-\delta$, the optimization path $\rvw_0,\rvw_1,...,\rvw_T$ is covered by the ball $B(\rvw_0,R)$ with $R = \frac{2n\sqrt{2\gamma}\sqrt{\l(\rvw_0)}}{\mu \delta}$.
\end{proof}

While the global landscape of the loss function $\l(\rvw)$ can be complex,  the conditions above allow us to find solutions within  a certain ball  around the initialization point $\rvw_0$.

\section{Concluding thoughts and  comments}

In this paper we have proposed a general framework for understanding generically non-convex landscapes and optimization of over-parameterized systems in terms of the \pls condition. We have argued that PL* condition generally holds on most but not all of the parameter space, which is sufficient for the existence of solutions and convergence of gradient-base methods to global minimizers. 
In contrast, it is not possible for loss landscapes of under-parameterized systems $\|\F(\rvw) - \rvy\|^2$ to satisfy \pls for any $\rvy$. 
We conclude with a number of comments and observations.

\paragraph{Linear and non-linear systems.}
A  remarkable property of over-parameterized non-linear systems discussed in this work is their strong resemblance to linear systems with respect to optimization by (S)GD, even as their dynamics remain nonlinear. 
In particular, optimization by gradient-based methods and proximity to global minimizers is controlled by non-linear condition numbers,
similarly to classical analyses of linear systems.  The key difference is that while  for linear systems the condition number is constant, in the non-linear case we need a uniform bound in a domain containing the optimization path. In contrast, the optimization properties of non-linear systems in the {\it under-parameterized} regime appear very different from those of linear systems. 
Furthermore, increasing the degree of over-parameterization generally improves conditioning just like it does for linear systems (cf.~\cite{chen2005condition} and the discussion in~\cite{poggio2019double}).
In particular, this  suggests that the effectiveness of optimization should improve, up to a certain limit,  with increased over-parameterization. 

\paragraph{Transition over the interpolation threshold.} 
Recognizing the power of over-parameterization has been a key insight stemming from the practice of deep learning. Transition to over-parameterized models -- over the interpolation threshold -- leads to a qualitative change in a range of system properties.  
 Statistically, over-parameterized systems enter a new interpolating regime, where increasing the number of parameters, even indefinitely to infinity, can improve generalization~\cite{belkin2019reconciling,Spigler_2019}. From the optimization point of view, over-parameterized system are generally easier to solve. There has been significant effort (continued in this work) toward understanding effectiveness of local methods in this setting~\cite{soltanolkotabi2018theoretical,du2018gradientshallow,mei2018mean,arora2019fine,ji2019polylogarithmic,oymak2019towards}. 

In this paper we note  another aspect of this  transition, to the best of our knowledge not addressed in the existing literature -- transition from local convexity to essential non-convexity. This relatively simple observation has significant consequences, indicating the need to depart from the machinery of convex analysis. Interestingly, our analyses suggest that this loss of local convexity is of little consequence for optimization, at least as far as gradient-based methods are concerned.

\paragraph{Transition from over-parameterization to under-parameterization along the optimization path.}

As discussed above, transition over the interpolation threshold occurs when the number of parameters in a variably parameterized system exceeds the number of constraints (corresponding to data points in typical ML scenarios). 
Over-parameterization does not refer to the number of parameters as such but to the difference between the number of parameters and the number of constraints. 
While some learning models are very large, with billions or even trillions parameters, they are often trained on equally large datasets and are thus not necessarily over-parameterized. Yet, it is still tempting to view these models through the lens of over-parameterization. While precise technical analysis is beyond the scope of this paper, we conjecture that transition between effective over-parameterization to under-parameterization happens along the optimization trajectory. Initially, the system behaves as over-parameterized but as the optimization process continues, it fails to reach zero. Mathematically it can be represented as our \plse condition. We speculate that for many realistic large models trained on big data the full optimization path lies within the \plse domain and hence, functionally, the analyses in this paper apply.

\paragraph{Condition numbers and optimization methods.} In this work we concentrate on optimization by gradient descent and SGD. Yet, for linear systems of equations and in many other settings, the importance of conditioning extends far beyond one specific optimization technique~\cite{burgisser2013condition}. We expect this to be case in the over-parameterized non-linear setting as well.  To give just one example, we expect that accelerated methods, such as the Nesterov's method~\cite{nesterov1983method} and its stochastic gradient extensions in the over-parameterized case~\cite{liu2018accelerating,vaswani2019fast} to have faster convergence rates for non-linear systems in terms of the condition numbers defined in this work.

\paragraph{Equations on manifolds.}
In this paper we consider systems of equations $\F(\rvw)=\rvy$ defined on Euclidean spaces and with Euclidean output.  A more general setting is to look for 
solutions of arbitrary systems of equations defined by a map between two Riemannian manifolds $\F:{\cal M} \to \cal{N}$.
In that case the loss function $\l$ needs to be defined on $\cal{N}$.
The over-parameterization corresponds to the case when dimension $\dim({\cal M}) > \dim({\cal N})$. While analyzing gradient descent requires some care on a manifold, most of the mathematical machinery, including the definitions of the \pls condition and the condition number associated to $\F$, is still applicable without significant change. In particular, as we discussed above (see Remark~\ref{remark:transform} in Section~\ref{subsec:uc_ball}), the condition number is preserved under ``well-behaved'' coordinate transformations. In contrast, this is not the case for the Hessian and thus manifold optimization analyses  based on {\it geodesic convexity} require knowledge about specific coordinate charts, such as those given by the exponential map. 

We note that manifold and structural assumptions on the weight vector $\rvw$ is a natural setting for addressing many problems in inference. In particular, the important class of convolutional  neural networks is an example of such a structural assumption on $\rvw$, which is made invariant to certain parallel transforms. Furthermore, there are many settings, e.g., robot motion planning, where the output of a predictor, $\rvy$, also belongs to a certain manifold.

\section*{Acknowledgements}
We thank Raef Bassily and Siyuan Ma for many earlier discussions about gradient methods and Polyak-{\L}ojasiewicz conditions and Stephen Wright for insightful comments and corrections.
The authors acknowledge support from the NSF, the Simons Foundation and a Google Faculty Research Award. 

\printbibliography

\newpage

\appendix
\section{Wide neural networks have no isolated local/global minima}\label{secapp:no_isolated}
In this section, we show that, for feedforward neural networks, if the network width is sufficiently large, there is no isolated local/global minima in the loss landscape.

Consider the following feedforward neural networks:
\begin{equation}\label{eq:feedforwardnn}
    f(\rmW;\rvx):= W^{(L+1)}\phi_L(W^{(L)}\cdots \phi_1(W^{(1)}\rvx)).
\end{equation}
Here, $\rvx\in \mathbb{R}^d$ is the input, $L$ is the number of hidden layers of the network. Let $m_l$ be the width of $l$-th layer, $l\in[L]$, and $m_0=d$ and $m_{L+1}=c$ where $c$ is the dimension of the network output.  $\rmW := \{W^{(1)},W^{(2)},\cdots ,W^{(L)},W^{(L+1)}\}$, with $W^{(l)}\in\mathbb{R}^{m_{l-1}\times m_l}$, is the collection of parameters, and $\phi_{l}$ is the activation function at each hidden layer. The minimal width of hidden layers, i.e., width of the network, is denoted as $m:=\min \{m_1,\cdots, m_L\}$. The parameter space is denoted as $\mathcal{M}$. Here, we further assume the loss function $\mathcal{L}(\rmW)$ is of the form: 
\begin{equation}\label{eqapp:loss}
    \mathcal{L}(\rmW) = \sum_{i=1}^n l(f(\rmW;\rvx_i),\rvy_i),
\end{equation}
where loss $l(\cdot,\cdot)$ is convex in the first argument, and $(\rvx_i,\rvy_i)\in \mathbb{R}^{d}\times \mathbb{R}^c$ is one of the $n$ training samples.

\begin{prop}[No isolated minima]\label{prop:no-iso}
Consider the feedforward neural network in Eq.(\ref{eq:feedforwardnn}). If the network width $m\ge 2c(n+1)^L$,  given a local (global) minimum $\rmW^*$ of the loss function Eq.(\ref{eqapp:loss}), there are always other local (global) minima in any neighborhood of $\rmW^*$.
\end{prop}

The main idea of this proposition is based on some of the intermediate-level results of the work~\cite{lederer2020no}. Before starting the proof, let's review some of the important concepts and results therein. To be consistent with the notation in this paper, we modify some of their notations.

\begin{defi}[Path constant]
Consider two parameters $\rmW,\rmV\in \mathcal{M}$. If there is a continuous function $h_{\rmW,\rmV}:[0,1]\to \mathcal{M}$ that satisfies $h_{\rmW,\rmV}(0)=\rmW$, $h_{\rmW,\rmV}(1)=\rmV$ and $t\mapsto\mathcal{L}(h_{\rmW,\rmV}(t))$ is constant, we say that $\rmW$ and $\rmV$ are {\it path constant} and write $\rmW \leftrightarrow \rmV$.
\end{defi}
Path constantness means the two parameters $\rmW$ and $\rmV$ are connected by a continuous path of parameters that is contained in a level set of the loss function $\mathcal{L}$.
\begin{lemma}[Transitivity]
$\rmW \leftrightarrow \rmV$ and $\rmV \leftrightarrow \rmU$ $\Rightarrow$ $\rmW \leftrightarrow \rmU$.
\end{lemma}
\begin{defi}[Block parameters]
Consider a number $s \in \{0,1,\cdots\}$ and a parameter $\rmW\in \mathcal{M}$. If 
\begin{eqnarray*}
& &\textrm{1. } W^{(1)}_{ji}=0 \ \textrm{for all } j> s;\\
& &\textrm{2. } W^{(l)}_{ij}=0 \ \textrm{for all } l\in [L] \textrm{ and } i> s, \textrm{ and for all } l\in [L] \textrm{ and } j> s;\\
& &\textrm{3. } W^{(L+1)}_{ij}=0 \ \textrm{for all } j> s,
\end{eqnarray*}
we call $\rmW$ an $s$-upper-block parameter of depth $L$.

\noindent Similarly, if 
\begin{eqnarray*}
& &\textrm{1. } W^{(1)}_{ji}=0 \ \textrm{for all } j\le m_1-s;\\
& &\textrm{2. } W^{(l)}_{ij}=0 \ \textrm{for all } l\in [L] \textrm{ and } i\le m_l-s, \textrm{ and for all } l\in [L] \textrm{ and } j\le m_{l-1}-s;\\
& &\textrm{3. } W^{(L+1)}_{ij}=0 \ \textrm{for all } j\le m_L-s,
\end{eqnarray*}
we call $\rmW$ an $s$-lower-block parameter of depth $L$. We denote the sets of the $s$-upper-block and $s$-lower-block parameters of depth $L$ by $\mathcal{U}_{s,L}$ and $\mathcal{V}_{s,L}$, respectively.

\end{defi}
The key result that we use in this paper is that every parameter is path constant to a block parameter, or more formally:
\begin{prop}[Path connections to block parameters~\cite{lederer2020no}]\label{prop_app:lederer}
For every parameter $\rmW\in \mathcal{M}$ and $s:=c(n+1)^L$ ($n$ is the number of training samples), there are $\overline{\rmW},\underline{\rmW}\in \mathcal{M}$ with $\overline{\rmW}\in \mathcal{U}_{s,L}$ and $\underline{\rmW}\in \mathcal{V}_{s,L}$ such that $\rmW\leftrightarrow \overline{\rmW}$ and $\rmW\leftrightarrow \underline{\rmW}$.
\end{prop}

The above proposition says that, if the network width $m$ is large enough, every parameter is path connected to both an upper-block parameter and a lower-block parameter, by continuous paths contained in a level set of the loss function.

Now, let's present the proof of Proposition~\ref{prop:no-iso}.
\begin{proof}[Proof of Proposition~\ref{prop:no-iso}]
Let $\rmW^*\in \mathcal{M}$ be an arbitrary local/global minimum of the loss function $\mathcal{L}(\rmW)$. According to Proposition~\ref{prop_app:lederer}, there exist an upper-block parameter $\overline{\rmW}\in \mathcal{U}_{s,L}\subset \mathcal{M}$ and $\underline{\rmW}\in \mathcal{V}_{s,L}\subset \mathcal{M}$ such that $\rmW\leftrightarrow \overline{\rmW}$ and $\rmW\leftrightarrow \underline{\rmW}$. Note that $\overline{\rmW}$ and $\underline{\rmW}$ are distinct, because $\mathcal{U}_{s,L}$ and $\mathcal{V}_{s,L}$ do not intersect except at zero due to $m\ge 2s$.
This means there must be parameters distinct from $\rmW^*$ that is connected to $\rmW^*$ via a continuous path contained in a level set of the loss function. Note that all the points (i.e., parameters) along this path have the same loss value as $\mathcal{L}(\rmW^*)$, hence are local/global minima. Therefore, $\rmW^*$ is not isolated (i.e, there are other local/global minima in any neighborhood of $\rmW^*$).
\end{proof}

\section{Proof of Proposition~\ref{prop:nolocalconx}}\label{prf:nolocalconx}
\begin{proof}
The Hessian matrix of a general loss function $\l(\F(\rvw))$ takes the form
\begin{align*}
    H_\l(\rvw) = D \F(\rvw)^T \frac{\partial^2\l}{\partial \F^2}(\rvw)D \F(\rvw) +\sum_{i=1}^n\frac{\partial\l}{\partial \F_i}(\rvw) H_{\F_i}(\rvw).
\end{align*}
Recall that $H_{\F_i}(\rvw)$ is the Hessian matrix of $i$-th output of $\F$ with respect to $\rvw$.

We consider the Hessian matrices of $\l$ around a global minimizer $\rvw^*$ of the loss function, i.e., solution of the system of equations. Specifically, consider the following two points $\rvw^*+\boldsymbol{\delta}$ and $\rvw^*-\boldsymbol{\delta}$, which are in a sufficiently small neighborhood of the minimizer $\rvw^*$. Then Hessian of loss at these two points are
\begin{eqnarray*}
 H_\l(\rvw^*+\boldsymbol{\delta}) &=& \underbrace{D \F(\rvw^*+\boldsymbol{\delta})^T\frac{\partial^2\l}{\partial \F^2}(\rvw^*+\boldsymbol{\delta}) D \F(\rvw^*+\boldsymbol{\delta})}_{A(\rvw^*+\boldsymbol{\delta})} +\sum_{i=1}^n \left(\frac{d}{d\rvw}\left(\frac{\partial \l}{\partial \F}(\rvw^*)\right)\boldsymbol{\delta}\right)_i H_{\F_i}(\rvw^*+\boldsymbol{\delta}) + o(\| \boldsymbol{\delta}\|),\\
 H_\l(\rvw^*-\boldsymbol{\delta}) &=& \underbrace{D \F(\rvw^*-\boldsymbol{\delta})^T\frac{\partial^2\l}{\partial \F^2}(\rvw^*-\boldsymbol{\delta}) D \F(\rvw^*-\boldsymbol{\delta})}_{A(\rvw^*-\boldsymbol{\delta})} -\sum_{i=1}^n \left(\frac{d}{d\rvw}\left(\frac{\partial \l}{\partial \F}(\rvw^*)\right)\boldsymbol{\delta}\right)_i H_{\F_i}(\rvw^*+\boldsymbol{\delta})  + o(\| \boldsymbol{\delta}\|).
\end{eqnarray*}
Note that both the terms $A(\rvw^*+\boldsymbol{\delta})$ and $A(\rvw^*-\boldsymbol{\delta})$ are matrices with rank at most $n$, since $D \F$ is of the size $n\times m$. 

By the assumption, at least one component $H_{\F_k}$ of the Hessian of $\F$ satisfies that the rank of $H_{\F_k}(\rvw^*)$ is greater than $2n$. By the continuity of the Hessian, we have that, if magnitude of $\boldsymbol{\delta}$ is sufficiently small, then the ranks of $H_{\F_k}(\rvw^*+\boldsymbol{\delta})$ and $H_{\F_k}(\rvw^*-\boldsymbol{\delta})$ are also greater than $2n$.

Hence, we can always find a unit length vector $\rvv\in \mathbb{R}^m$ s.t.
\begin{equation}
    \rvv^TA(\rvw^*+\boldsymbol{\delta})\rvv = \rvv^TA(\rvw^*-\boldsymbol{\delta})\rvv = 0,
\end{equation}
but
\begin{equation}
    \rvv^T H_{\F_k}(\rvw^*+\boldsymbol{\delta})\rvv \ne 0, \quad  \rvv^T H_{\F_k}(\rvw^*-\boldsymbol{\delta})\rvv \ne 0.
\end{equation}
Consequently the vector $\langle \rvv^T H_{\F_1}(\rvw^*+\boldsymbol{\delta})\rvv, \ldots ,\rvv^T H_{\F_n}(\rvw^*+\boldsymbol{\delta})\rvv\rangle \neq \mathbf{0}$ and $\langle \rvv^T H_{\F_1}(\rvw^*-\boldsymbol{\delta})\rvv, \ldots ,\rvv^T H_{\F_n}(\rvw^*-\boldsymbol{\delta})\rvv\rangle \neq \mathbf{0}$.

With the same $\rvv$, we have
\begin{align}
 \label{eq:hessiandirec1}   \rvv^T  H_\l(\rvw^*+\boldsymbol{\delta})\rvv &= \sum_{i=1}^n \left(\frac{d}{d\rvw}\left(\frac{\partial \l}{\partial \F}(\rvw^*)\right)\boldsymbol{\delta}\right)_i \rvv^T H_{\F_i}(\rvw^*+\boldsymbol{\delta})\rvv + o(\| \boldsymbol{\delta}\|), \\
     \rvv^T  H_\l(\rvw^*-\boldsymbol{\delta})\rvv &=  \label{eq:hessiandirec2} -\sum_{i=1}^n\left(\frac{d}{d\rvw}\left(\frac{\partial \l}{\partial \F}(\rvw^*)\right)\boldsymbol{\delta}\right)_i\rvv^T H_{\F_i}(\rvw^*-\boldsymbol{\delta})\rvv + o(\| \boldsymbol{\delta}\|).
\end{align}
In the following, we show that, for sufficiently small $\boldsymbol{\delta}$, $\rvv^T  H_\l(\rvw^*+\boldsymbol{\delta})\rvv$ and $\rvv^T  H_\l(\rvw^*-\boldsymbol{\delta})\rvv$ can not be non-negative simultaneously, which immediately implies that $H_\l$ is not positive semi-definite in the close neighborhood of $\rvw^*$, hence $\l$ is not locally convex at $\rvw^*$. 

Specifically,
with the condition $\frac{d}{d\rvw}\left(\frac{\partial \l}{\partial \F}(\rvw^*)\right)\neq \mathbf{0}$,  for Eq.(\ref{eq:hessiandirec1}) and Eq.(\ref{eq:hessiandirec2}) we have the following cases:\\
{\it Case 1} : If $\sum_{i=1}^n\left(\frac{d}{d\rvw}\left(\frac{\partial \l}{\partial \F}(\rvw^*)\right)\boldsymbol{\delta}\right)_i \rvv^T H_{\F_i}(\rvw^*+\boldsymbol{\delta})\rvv < 0$, then directly $\rvv^T  H_\l(\rvw^*+\boldsymbol{\delta})\rvv < 0$ if $\boldsymbol{\delta}$ is small enough which completes the proof.\\ 
{\it Case 2} : Otherwise if  $ \sum_{i=1}^n \left(\frac{d}{d\rvw}\left(\frac{\partial \l}{\partial \F}(\rvw^*)\right)\boldsymbol{\delta}\right)_i \rvv^T H_{\F_i}(\rvw^*+\boldsymbol{\delta})\rvv > 0$, by the continuity of each $H_{\F_i}(\cdot)$, we have
\begin{align*}
    & -\sum_{i=1}^n\left(\frac{d}{d\rvw}\left(\frac{\partial \l}{\partial \F}(\rvw^*)\right)\boldsymbol{\delta}\right)_i \rvv^T H_{\F_i}(\rvw^*-\boldsymbol{\delta})\rvv \\
    &= - \sum_{i=1}^n\left(\frac{d}{d\rvw}\left(\frac{\partial \l}{\partial \F}(\rvw^*)\right)\boldsymbol{\delta}\right)_i \rvv^T H_{\F_i}(\rvw^*+\boldsymbol{\delta})\rvv  +\sum_{i=1}^n\left(\frac{d}{d\rvw}\left(\frac{\partial \l}{\partial \F}(\rvw^*)\right)\boldsymbol{\delta}\right)_i \rvv^T (H_{\F_i}(\rvw^*+\boldsymbol{\delta})-H_{\F_i}(\rvw^*-\boldsymbol{\delta}))\rvv \\
    &= - \sum_{i=1}^n \left(\frac{d}{d\rvw}\left(\frac{\partial \l}{\partial \F}(\rvw^*)\right)\boldsymbol{\delta}\right)_i \rvv^T H_{\F_i}(\rvw^*+\boldsymbol{\delta})\rvv + O(\epsilon) \\
    &< 0,
\end{align*}
when $\boldsymbol{\delta}$ is small enough.
\paragraph{Note:} If $\sum_{i=1}^n\left(\frac{d}{d\rvw}\left(\frac{d\l}{d\F}(\rvw^*)\right)\boldsymbol{\delta}\right)_i \rvv^T H_{\F_i}(\rvw^*+\boldsymbol{\delta})\rvv = 0$, we can always adjust $\boldsymbol{\delta}$ a little so that it turns to case 1 or 2.\\

In conclusion, with certain $\boldsymbol{\delta}$ which is arbitrarily small, either $\rvv^T  H_\l(\rvw^*+\boldsymbol{\delta})\rvv$ or $\rvv^T  H_\l(\rvw^*-\boldsymbol{\delta})\rvv$ has to be negative which means $\l(\rvw)$ has no convex neighborhood around $\rvw^*$.
\end{proof}

\section{Small Hessian is a feature of certain large models}\label{appsec:small_hessian}

Here, we show that the small Hessian spectral norm is not an strong condition. In fact, it is a mathematical freebie as long as the model has certain structure and is large enough. For example, if a neural network has a linear output layer and is wide enough, its Hessian spectral norm can be arbitrarily small, see~\cite{liu2020linearity}.

In the following, let's consider an illustrative example. 
Let the model $f$ be a linear combination of $m$ independent sub-models, 
\begin{equation}\label{eq:generalsparse}
    f(\rvw;\rvx) = \frac{1}{s(m)}\sum_{i=1}^m v_i \alpha_{i}(\rvw_i;\rvx),
\end{equation}
where  $\rvw_i$ is the trainable parameters of the sub-model $\alpha_i$, $i\in [m]$, and $\frac{1}{s(m)}$ is a scaling factor that depends on $m$. The sub-model weights $v_i$ are independently randomly chosen from $\{-1,1\}$ and not trainable.
The parameters of the model $f$ are the concatenation of sub-model parameters: $\rvw := (\rvw_1, \cdots, \rvw_m)$ with $\rvw_i\in \mathbb{R}^p$. We make the following mild assumptions.

\begin{assumption}\label{assu:1}
For simplicity, we assume that the sub-models $\alpha_i(\rvw_i,\rvx)$ have the same structure but different initial parameters $\rvw_{i,0}$ due to random initialization. We further assume each sub-model has a $\Theta(1)$ output, and is second-order differentiable and $\beta$-smooth.
\end{assumption} 

Due to the randomness of the sub-model weights $v_i$, the scaling factor $\frac{1}{s(m)} = o(1)$ w.r.t. the size $m$ of the model $f$ (e.g, $\frac{1}{s(m)} = \frac{1}{\sqrt{m}}$ for neural networks~\cite{jacot2018neural,liu2020linearity}). 

The following theorem states that, as long as the model size $m$ is sufficiently large, the Hessian spectral norm of the model $f$ is arbitrarily small.
\begin{thm}
Consider the model $f$ defined in Eq.(\ref{eq:generalsparse}). Under Assumption \ref{assu:1}, the spectral norm of the Hessian of model $f$ satisfies:
\begin{equation}
    \|H_f\|\le\frac{\beta}{s(m)}.
\end{equation}
\end{thm}
\begin{proof}
An entry $(H_f)_{jk}$, $j,k\in \mathbb{R}^{m\times p}$, of the model Hessian matrix is
\begin{equation}\label{eq:hessian_simple_model}
    (H_f)_{jk} = \frac{1}{s(m)}\sum_{i=1}^m v_i\frac{\partial^2\alpha_i}{\partial w_j\partial w_k} =:\frac{1}{s(m)}\sum_{i=1}^m v_i(H_{\alpha_i})_{jk},
\end{equation}
with $H_{\alpha_i}$ being the Hessian matrix of the sub-model $\alpha_i$.
Because the parameters of $f$ is the concatenation of sub-model parameters and the sub-models share no common parameters, in 
the summation of Eq.(\ref{eq:hessian_simple_model}), there must be at most one non-zero term (non-zero only when $w_j$ and $w_k$ belong to the same model $\alpha_i$. Thus, the Hessian spectral norm can be bounded by 
\begin{equation}
    \|H_f\|\le \frac{1}{s(m)} \max_{i\in [m]}|v_i|\cdot\|H_{\alpha_i}\| \le \frac{1}{s(m)}\cdot \beta.
\end{equation}
\end{proof}

\section{An illustrative example of composition models}\label{secapp:compostion}

Consider the model $h = g\circ f$ as a composition of 2 random Fourier feature models $f: \mathbb{R} \to \mathbb{R}$ and $g: \mathbb{R} \to \mathbb{R}$ with:
\begin{equation}
    f(\rvu;x) = \frac{1}{\sqrt{m}}\sum_{k=1}^{m} \left(u_k \cos (\omega_kx) + u_{k+m}\sin (\omega_k x)\right),
\end{equation}
\begin{equation}
        g(\rva; z) = \frac{1}{\sqrt{m}}\sum_{k=1}^{m} \left(a_k \cos (\nu_kz) + a_{k+m}\sin (\nu_k z)\right).
\end{equation}
Here we set the frequencies $\omega_k \sim \mathcal{N}(0,1)$ and $\nu_k \sim \mathcal{N}(0,n^2)$,  for all $k\in [m]$, to be fixed. The trainable parameters of the model are $(\rvu, \rva)\in \mathbb{R}^{2m}\times \mathbb{R}^{2m}$. It's not hard to see that the model $h(x) = g\circ f(x)$ can be viewed as a 4-layer "bottleneck" neural network where the second hidden layer has only one neuron, i.e. the output of $f$.

For simplicity, we let the input $x$ be $1$-dimensional, and denote the training data as $x_1,...,x_n \in \mathbb{R}$. We consider the case of both sub-models, $f$ and $g$, are large, especially $m\to \infty$.

By Proposition~\ref{prop:composition}, the tangent kernel matrix of $h$, i.e. $K_h$, can be decomposed into the sum of two positive  semi-definite matrices, and the uniform conditioning of $K_h$ can be guaranteed if one of them is uniformly conditioned, as demonstrated in Eq.(\ref{eq:composition}). In the following, we show $K_g$ is uniformly conditioned by making $f$ well separated.

We assume the training data are not degenerated and the parameters $\rvu$ are randomly initialized. This makes sure the initial outputs of $f$, which are the initial inputs of $g$, are not degenerated, with high probability. For example, let $\min_{i\ne j}|x_i-x_j| \ge \sqrt{2}$ and initialize $\rvu$ by $\mathcal{N}(0,100R^2)$ with a given number $R > 0$. Then, we have 
\begin{equation}
    f(\rvu_0;x_i) - f(\rvu_0;x_j) \sim \mathcal{N}\left(0,200R^2-200R^2 e^{\frac{-|x_i-x_j|^2}{2}}\right), \ \forall i,j \in [n].\nonumber
\end{equation}
Since $\min_{i\ne j}|x_i-x_j| \ge \sqrt{2}$, the variance $Var:=200R^2-200R^2 e^{\frac{-|x_i-x_j|^2}{2}} > 100R^2$. For this Gaussian distribution, we have, with probability at least $0.96$, that
\begin{equation}\label{eqapp:bottle_upper}
    \min_{i\ne j}|f(\rvu_0;x_i) - f(\rvu_0;x_j)| > 2R.
\end{equation}

For this model, the partial tangent kernel $K_g$ is the Gaussian kernel in the limit of $m\to \infty$, with the following entries:
\begin{equation}
    K_{g,ij}(\rvu)= \exp\left(-\frac{n^2|f(\rvu;x_i)-f(\rvu;x_j)|^2}{2}\right).\nonumber
\end{equation}
By the Gershgorin circle theorem, its smallest eigenvalue is lower bounded by:
\begin{equation*}
    \inf_{\rvu\in \mathcal{S}}\lambda_{min}(K_g(\rvu)) \ge 1-(n-1)\exp\left(-\frac{n^2\rho(\mathcal{S})^2}{2}\right),
\end{equation*}
where $\rho(\mathcal{S}):= \inf_{\rvu\in \mathcal{S}}\min_{i\ne j}|f(\rvu;x_i)-f(\rvu;x_j)|$ is the minimum separation between the outputs of $f$, i.e., inputs of $g$ in $\mathcal{S}$.  If $\rho(\mathcal{S})\ge \sqrt{2\ln (2n-2)}/n$, then we have $\inf_{\rvu\in \mathcal{S}}\lambda_{min}(K_g(\rvu))\ge 1/2$. Therefore, we see that the uniform conditioning of $K_g$, hence that of $K_h$, is controlled by the separation between the inputs of $g$, i.e., outputs of $f$.

Within the ball $B((\rvu_0,\rva_0),R):=\{ (\rvu,\rva) : \| \rvu - \rvu_0\|^2 + \| \rva - \rva_0\|^2 \leq R^2\}$ with an arbitrary radius $R>0$, the outputs of $f$ are always well separated, given that the initial outputs of $f$ are separated by $2R$ as already discussed above. This is because 
\begin{equation}
    |f(\rvu;x)-f(\rvu_0;x)| = \frac{1}{\sqrt{m}}\left|\sum_{k=1}^m\left((u_k-u_{0,k})\cos (\omega_k x)+(u_{k+m}-u_{0,k+m})\sin (\omega_k x)\right)\right| \le \|\rvu-\rvu_0\|\le R,\nonumber
\end{equation}
which leads to $\rho(B((\rvu_0,\rva_0),R)) > R$ by Eq.(\ref{eqapp:bottle_upper}). By choosing $R$ appropriately, to be specific, $R\geq \sqrt{2\ln(2n-2)}/n$, the uniform conditioning of $K_g$ is satisfied.


Hence, we see that composing large non-linear models may make the tangent kernel no longer constant, but the uniform conditioning of the tangent kernel can remain.

\section{Wide CNN and ResNet satisfy PL$^*$ condition}\label{appsec:cnn_resnet}
In this section, we will show that wide Convolutional Neural Networks (CNN) and Residual Neural Networks (ResNet) also satisfy the PL$^*$ condition. 

The CNN is defined as follows:

\begin{align}\label{eq:cnnlayer}
 &\alpha^{(0)} = \rvx, \nonumber\\
    &\alpha^{(l)}  = \sigma_l\left(\frac{1}{\sqrt{m_{l-1}}} W^{(l)} \ast \alpha^{(l-1)}\right), \ \forall l = 1,2,\ldots ,L,\nonumber\\
 &f(\rmW;\rvx) = \frac{1}{\sqrt{m_L}} \langle W^{(L+1)}, \alpha^{(L+1)}\rangle,
\end{align} 
where $*$ is the convolution operator (see the definition below) and $\langle \cdot,\cdot\rangle$ is the standard matrix inner product. 

Compared to the definition of fully-connected neural networks in Eq.(\ref{eq:generalnn}), the $l$-th hidden neurons $\alpha^{(l)}\in \mathbb{R}^{m_l\times Q}$ is a matrix where $m_l$ is the number of channels and $Q$ is the number of pixels, and $W^{(l)} \in \mathbb{R}^{K\times m_{l} \times m_{l-1}}$ is an order 3 tensor where $K$ is the filter size except that $W^{(L+1)} \in \mathbb{R}^{m_{L}\times Q}$. 

For the simplicity of the notation, we give the definition of convolution operation for 1-D CNN in the following. We note that it's not hard to extend to higher dimensional CNNs and one will find that our analysis still applies.

 \begin{equation}\label{eq:convolution}
     \left(W^{(l)} \ast \alpha^{(l-1)}\right)_{i,q} = \sum_{k=1}^K \sum_{j=1}^{m_{l-1}} W^{(l)}_{k,i,j} \alpha^{(l-1)}_{j,q+k-\frac{K+1}{2}}, \ i\in [m_l], q\in [Q].
 \end{equation}
 
 The ResNet is similarly defined as follows:
 
 \begin{align}\label{eq:reslayer}
 &\alpha^{(0)} = \rvx, \nonumber\\
 &\alpha^{(l)}= \sigma_{l}\left(\frac{1}{\sqrt{m_{l-1}}}W^{(l)}\alpha^{(l-1)}\right) + \alpha^{(l-1)}, \  \ \forall l = 1,2,\cdots, L+1,\nonumber\\
 &f(\rmW;\rvx) = \alpha^{(L+1)}.
\end{align} 

We see that the ResNet is the same as a fully connected neural network, Eq.~(\ref{eq:generalnn}), except that the activations $\alpha^{(l)}$ has an extra additive term $\alpha^{(l-1)}$ from the previous layer, interpreted as skip connection.

\begin{remark}
This definition of ResNet differs from the standard ResNet architecture in~\cite{he2016deep}
that the skip connections are at every layer, instead of every two layers. One will find that the same analysis can be easily generalized to cases where skip connections are at every two or more layer.  The same definition, up to a scaling factor, was also theoretically studied in \cite{du2018gradientdeep}.
\end{remark}

By the following theorem, we have an upper bound for the Hessian spectrum norm of the CNN and ResNet, similar to Theorem~\ref{cor:different_width} for fully-connected neural networks. 

\begin{thm}[Theorem 3.3 of \cite{liu2020linearity}] \label{thm:different_width_res_cnn}
Consider a neural network $f(\rmW;\rvx)$ of the form Eq.(\ref{eq:cnnlayer}) or Eq.(\ref{eq:reslayer}). Let $m$ be the minimum of the hidden layer widths, i.e., $m= \min_{l\in[L]} m_l$. Given any fixed $R>0$, and any  $\rmW \in B(\rmW_0,R):= \{\rmW: \|\rmW - \rmW_0\| \le R\}$,   with high probability over the initialization, the Hessian spectral norm satisfies the following:
\begin{equation}\label{eq:hessian_bound_cnn_res}
    \|H_f(\rmW)\| = \tilde{O}\left(R^{3L}/{\sqrt{m}}\right). 
\end{equation}
\end{thm}

Using the same analysis in Section~\ref{sec:wide_nn_pl*}, we can get a similar result with Theorem~\ref{thm:deepnnpl} for CNN and ResNet to show they also satisfy PL$^*$ condition:

\begin{thm}[Wide CNN and ResNet satisfy PL$^{*}$ condition]\label{thm:deepnnpl_res_cnn}
Consider the neural network $f(\rmW;\rvx)$ in Eq.(\ref{eq:cnnlayer}) or Eq.(\ref{eq:reslayer}), and a random parameter setting $\rmW_0$ such that  each element of $W_0^{(l)}$ for $l\in[L+1]$ follows $\mathcal{N}(0,1)$. Suppose that the last layer activation $\sigma_{L+1}$ satisfies $|\sigma'_{L+1}(z)| \ge \rho > 0$ and that  $\lambda_0:=\lambda_{min}(K(\rmW_0)) >0$.  For any $\mu \in (0,\lambda_0\rho^2)$, if the width of the network 
\begin{equation}\label{eq:width_cnn_res}
    m = \tilde{\Omega}\left(\frac{nR^{6L+2}}{(\lambda_0-\mu\rho^{-2})^2}\right),
\end{equation}
then $\mu$-PL$^*$ condition holds for square loss in the ball $B(\rmW_0,R)$.
\end{thm}

\section{Proof for convergence under PL$^*$ condition}\label{secapp:plstar}
In Theorem~\ref{thm:plstar}, the convergence of gradient descent is established in the case of square loss function. In fact, similar results (with a bit modification) hold for general loss functions. In the following theorem, we provide the convergence under PL$^*$ condition for general loss functions. Then, we prove Theorem~\ref{thm:plstar} and \ref{thm:plstar_general_loss} together.

\begin{thm}\label{thm:plstar_general_loss}
Suppose the loss function $\l(\rvw)$ (not necessarily the square loss) is $\beta$-smooth and  satisfies the $\mu$-PL$^*$ condition in the ball $B(\rvw_0,R) := \{\rvw\in \mathbb{R}^m : \|\rvw-\rvw_0\| \le R\}$ with $R = \frac{2\sqrt{2\beta \l(\rvw_0)}}{\mu}$. Then we have the following:

\noindent (a) Existence of a solution: There exists a solution (global minimizer of $\l$) $\rvw^* \in B(\rvw_0,R)$, such that $\F(\rvw^{*})=\rvy$.

\noindent (b) Convergence of GD: Gradient descent with a step size $\eta \le 1/\sup_{\rvw\in B(\rvw_0,R)}\|H_\l(\rvw)\|_2$ converges to a global solution in  $B(\rvw_0,R)$, with an exponential (also known as linear) convergence rate: 
\begin{equation}
    \l(\rvw_t) \le \left(1-{\kappa_{\l,\F}^{-1}(B(\rvw_0,R))}\right)^t\l(\rvw_0).
\end{equation}
where the condition number $\kappa_{\l,\F}(B(\rvw_0,R)) = \frac{1}{\eta\mu}$.
\end{thm}
\begin{proof}
Let's start with Theorem \ref{thm:plstar_general_loss}.
We prove this theorem by induction. The induction hypothesis is that, for all $t\ge 0$, $\rvw_t$ is within the ball $B(\rvw_0,R)$ with $R = \frac{2\sqrt{2\beta \l(\rvw_0)}}{\mu}$, and 
\begin{equation}\label{eqapp:plconvergence}
    \l(\rvw_t) \le \left(1-\eta\mu\right)^t\l(\rvw_0).
\end{equation}
In the base case, where $t=0$, it is trivial that $\rvw_0\in B(\rvw_0,R)$ and that $\l(\rvw_t) \le \left(1-\eta\mu\right)^0\l(\rvw_0).$

Suppose that, for a given $t\ge 0$, $\rvw_t$ is in the ball $B(\rvw_0,R)$ and Eq.(\ref{eqapp:plconvergence}) holds. As we show separately below (in Eq.(\ref{eqapp:radius}) we have $\rvw_{t+1} \in B(\rvw_0,R)$.
Hence we see that $\l$ is $(\sup_{\rvw\in B(\rvw_0,R)}\|H_\l(\rvw)\|_2)$-smooth in $B(\rvw_0,R)$. By  definition of $\eta  = 1/\sup_{\rvw\in B(\rvw_0,R)}\|H_L(\rvw)\|_2$ and, consequently, $\l$ is $1/\eta$-smooth. Using that we obtain the following:


\begin{align}
    \l(\rvw_{t+1}) -\l(\rvw_t) - \nabla\l(\rvw_t)(\rvw_{t+1}-\rvw_t) 
    &\leq \frac{1}{2}\sup_{\rvw\in B(\rvw_0,R)}\|H_\l\|_2 \| \rvw_{t+1}-\rvw_t\|^2 \nonumber \\
    &= \frac{1}{2\eta} \| \rvw_{t+1}-\rvw_t\|^2.
\end{align}

Taking $\rvw_{t+1}-\rvw_t = -\eta \nabla \l(\rvw_t)$ and by $\mu$-PL$^*$ condition at point $\rvw_t$,  we have

\begin{eqnarray}
\l(\rvw_{t+1})-\l(\rvw_t) \le -\frac{\eta}{2}\|\nabla \l(\rvw_t)\|^2 \le -\eta\mu \l(\rvw_t).
\end{eqnarray}

Therefore, 
\begin{equation}
    \l(\rvw_{t+1})\le(1-\eta\mu) \l(\rvw_t)\leq  \left(1-\eta\mu\right)^{t+1}\l(\rvw_0).
\end{equation}

To prove $\rvw_{t+1} \in B(\rvw_0,R)$, by the fact that $\l$ is $\beta$-smooth, we have
\begin{eqnarray}
 \|\rvw_{t+1}-\rvw_0\| &=&\eta \|\sum_{i=0}^t\nabla \l(\rvw_{i})\| \nonumber\\
 &\le& \eta \sum_{i=0}^t\|\nabla \l(\rvw_{i})\|\nonumber\\
 &\le&\eta \sum_{i=0}^t \sqrt{2\beta (\l(\rvw_i)-\l(\rvw_{i+1}))}\nonumber\\
 &\le& \eta \sum_{i=0}^t \sqrt{2\beta \l(\rvw_i)}\nonumber\\
 &\le&\eta \sqrt{2\beta}\left(\sum_{i=0}^t\left(1-\eta\mu\right)^{i/2}\right)\sqrt{\l(\rvw_0)}\nonumber\\
 &\le& \eta \sqrt{2\beta}\sqrt{\l(\rvw_0)} \frac{1}{1-\sqrt{1-\eta\mu}}\nonumber\\
 &\le& \frac{2\sqrt{2\beta}\sqrt{\l(\rvw_0)}}{\mu}\nonumber\\
 &=& R.\label{eqapp:radius}
\end{eqnarray}
Thus, $\rvw_{t+1}$ resides in the ball $B(\rvw_0,R)$.



\noindent By the principle of induction, the hypothesis is true.

\noindent Now, we prove Theorem \ref{thm:plstar}, i.e., the particular case of square loss $\l(\rvw) = \frac{1}{2}\|\F(\rvw)-\rvy\|^2$. In this case, 
\begin{eqnarray}
    \nabla \l (\rvw) &=& (\F(\rvw)-\rvy)^TD \F(\rvw).
\end{eqnarray}
Hence, in Eq.(\ref{eqapp:radius}), we have the following instead:
\begin{eqnarray}
\|\rvw_{t+1}-\rvw_0\| &\le& \eta \sum_{i=0}^t\|\nabla \l(\rvw_{i})\|\nonumber\\
&\le&\eta \sum_{i=0}^t\|D \F(\rvw_{i})\|_2 \|\F(\rvw_{i})-\rvy)\|\nonumber\\
&\le&\eta L_\F\left(\sum_{i=0}^t (1-\mu/\beta_\F)^{i/2}\right)\|\F(\rvw_0)-\rvy)\|\nonumber\\
&\le& \eta L_\F \cdot \frac{2}{\eta \mu}\|\F(\rvw_0)-\rvy)\|\nonumber\\
&=& \frac{2L_\F\|\F(\rvw_0)-\rvy)\|}{\mu}.
\end{eqnarray}
Also note that, for all $t>0$, $\|H_\l(\rvw_t)\|_2 \le L_\F^2 + \beta_\F\cdot \|\F(\rvw_0)-\rvy\|$, since $\|\F(\rvw_t)-\rvy\| \le \|\F(\rvw_0)-\rvy\|$. Hence, the step size $\eta=1/L_F^2 + \beta_\F\cdot \|\F(\rvw_0)-\rvy\|$ is valid.
\end{proof}

\section{Proof of Theorem~\ref{thm:sgd_plstar}}\label{prf:sgd_plstar}
\begin{proof}
We first aggressively assume that the $\mu$-PL$^*$ condition holds in the whole parameter space $\mathbb{R}^m$. We will find that the condition can be relaxed to hold only in the ball $B(\rvw_0,R)$.

Following a similar analysis to the proof of Theorem 1 in \cite{bassily2018exponential}, by the $\mu$-PL$^*$ condition, we can get that for any mini-bath size $s$, the mini-batch SGD with step size $\eta^*(s) := \frac{n\mu}{n\lambda(n^2\lambda+\mu(s-1))}$ has an exponential convergence rate for all $t>0$:

\begin{align}
    \mathbb{E}[\l(\rvw_{t})]  \leq \left(1-\frac{\mu s \eta^*(s)}{n}\right)^t \l(\rvw_{0}).
\end{align}
Moreover, the expected length of each step is bounded by 
\begin{align*}
   \mathbb{E}\left[ \| \rvw_{t+1}-\rvw_t \| \right] &= \eta^* \mathbb{E}\left[ \|\sum_{j=1}^s \nabla\ell_{i_t^{(j)}}(\rvw_t)\| \right] \\
    &\leq \eta^*  \sum_{j=1}^s \mathbb{E}\left[\| \nabla \ell_{i_t^{(j)}}(\rvw_t)\|\right] \\
    &\leq \eta^*  \sum_{j=1}^s 
    \mathbb{E}\left[ \sqrt{2\lambda \ell_{i_t^{(j)}}} \right] \\
    &\leq \eta^* s  \mathbb{E}\left[ \sqrt{2\lambda\l(\rvw_t)}\right] \\
    &\leq \eta^*s\sqrt{2\lambda\mathbb{E}[\l(\rvw_t)]}\\
    &\leq \eta^*s \sqrt{2\lambda} \left(1-\frac{\mu s  \eta^*}{n}\right)^{t/2}\sqrt{\l(\rvw_0)},
\end{align*}
where we use $\{i_t^{(1)},i_t^{(2)},...,i_t^{(s)}\}$ to denote a random mini-batch of the dataset at step $t$.

Then the expectation of the length of the whole optimization path is bounded by
\begin{align*}
    \mathbb{E}\left[\sum_{i=0}^\infty \| \rvw_{i+1} - \rvw_i\| \right] &= \sum_{i=0}^\infty\mathbb{E}\| \rvw_{i+1} - \rvw_i\|\\
    &= \sum_{i=0}^\infty\eta^*s \sqrt{2\lambda} \left(1-\frac{\mu s  \eta^*}{n}\right)^{t/2}\sqrt{\l(\rvw_0)} = \frac{2n\sqrt{2\lambda}\sqrt{\l(\rvw_0)}}{\mu}.
\end{align*}
By Markov's inequality, we have, with probability at least $1-\delta$, the length of the path is shorter than $R$, i.e.,
\begin{align*}
    \sum_{i=0}^\infty \| \rvw_{i+1} - \rvw_i\| \leq \frac{2n\sqrt{2\lambda}\sqrt{\l(\rvw_0)}}{\mu \delta} = R.
\end{align*}
This means that, with probability at least $1-\delta$, the whole path is covered by the ball $B(\rvw_0,R)$, namely, for all $t$,
\begin{align*}
    \| \rvw_t - \rvw_0\| \leq \sum_{i=0}^{t-1} \| \rvw_{i+1} - \rvw_i \| \leq R.
\end{align*}
For those events when the whole path is covered by the ball, we can relax the satisfaction of the PL$^*$ condition from the whole space to the ball.
\end{proof}

\end{document}